%% file: nystrom-main.tex
\newif\ifisarxiv
\isarxivtrue

\documentclass{article}

\usepackage{multibib}
\usepackage{microtype}
\usepackage{graphicx}
\usepackage{subfigure}
\usepackage{booktabs} % for professional tables
\usepackage{amsmath}
\usepackage{amssymb}
\usepackage{amsthm}
\usepackage{mathabx}
\usepackage{color}
\usepackage{cancel}
\usepackage{wrapfig}
\usepackage{caption}
\usepackage{xfrac}
\usepackage{algorithm}
\usepackage{algorithmic}
\usepackage{mdframed}
\usepackage{hyperref}
\usepackage{cleveref}
\usepackage{xcolor}

\hypersetup{%draft, % TODO: remove "draft" option before final version after making sure no hyperlinks are broken due to pagebreaks.
	colorlinks,
	linkcolor={red!40!gray},
	citecolor={blue!40!gray},
	urlcolor={blue!70!gray}
}

\newtheorem*{rep@theorem}{\rep@title}
\newcommand{\newreptheorem}[2]{%
	\newenvironment{rep#1}[1]{%
		\def\rep@title{#2 \ref{##1}}%
		\begin{rep@theorem}}%
		{\end{rep@theorem}}}
\newreptheorem{lemma}{Lemma'}
\newreptheorem{definition}{Definition'}
\newreptheorem{proposition}{Proposition'}
\newreptheorem{theorem}{Theorem'}
\newreptheorem{corollary}{Corollary'}

\input{shortdefs}

\ifisarxiv
\usepackage{fullpage}
\usepackage{natbib}
\else
\usepackage[final]{sty/neurips2020/neurips_2020}

\title{Improved guarantees and a multiple-descent curve for %the CSSP
  Column Subset Selection %Problem
  and the Nystr\"om method}
\author{%
          Micha{\l } Derezi\'{n}ski \\
  Department of Statistics\\
  University of California, Berkeley\\
  \texttt{mderezin@berkeley.edu}\\
  \And
{Rajiv Khanna} \\
  Department of Statistics\\
  University of California, Berkeley\\
  \texttt{rajivak@berkeley.edu}
  \And
   Michael W. Mahoney\\
  ICSI and Department of Statistics\\
  University of California, Berkeley\\
  \texttt{mmahoney@stat.berkeley.edu}
  }
\fi
\begin{document}

\ifisarxiv
\title{Improved guarantees and a multiple-descent curve for\\ %the CSSP
  Column Subset Selection %Problem
  and the Nystr\"om method}
  \author{%
          \textbf{Micha{\l } Derezi\'{n}ski} \\
  Department of Statistics\\
  University of California, Berkeley\\
  \texttt{mderezin@berkeley.edu}\\
  \and
\textbf{Rajiv Khanna} \\
  Department of Statistics\\
  University of California, Berkeley\\
  \texttt{rajivak@berkeley.edu}
  \and
   \textbf{Michael W. Mahoney}\\
  ICSI and Department of Statistics\\
  University of California, Berkeley\\
  \texttt{mmahoney@stat.berkeley.edu}
  }

\else\fi

\ifisarxiv
  \date{}
\fi
  \maketitle

\begin{abstract}
The Column Subset Selection Problem (CSSP) and the Nystr\"om method
are among the leading tools for constructing small low-rank
approximations of large datasets in machine learning and scientific
computing.  
A fundamental question in this area is: how well can a data subset of
size $k$ compete with the best rank $k$ approximation?
We develop techniques which exploit spectral properties of the data
matrix to obtain improved approximation guarantees which go beyond the
standard worst-case analysis. 
Our approach leads to significantly better bounds for datasets with
known rates of singular value decay, e.g., polynomial or exponential
decay.  
Our analysis also reveals an intriguing phenomenon: the approximation
factor as a function of $k$ may exhibit multiple peaks and valleys,
which we call a multiple-descent curve.  
A lower bound we establish shows that this behavior is not an artifact
of our analysis, but rather it is an inherent property of the CSSP and
Nystr\"om tasks.
Finally, using the example of a radial basis function (RBF)
kernel, we show that both our improved bounds and the multiple-descent
curve can be observed on real datasets simply by varying the 
RBF parameter.  
\end{abstract}

\section{Introduction}
\label{s:intro}
We consider the task of selecting a small but representative sample of
column vectors from a large matrix. Known as the Column Subset Selection
Problem (CSSP), this is a well-studied combinatorial optimization task with
many applications in machine learning  
\citep[e.g., feature selection, see][]{Guyon03FeatureSelection,BoutsidisMD08},
scientific computing \citep[e.g.,][]{Chan92RankRevealingQR,Drineas08CUR} and signal processing
\citep[e.g.,][]{Balzano10MatchedSubspace}. In a 
commonly studied variant of this task, we aim to minimize the
squared error of projecting all columns of the matrix onto the
subspace spanned by the chosen column subset.
\begin{definition}[CSSP]
  Given an $m\times n$ matrix $\A$, pick a set $S\subseteq\{1,...,n\}$ of 
  $k$ column indices, to minimize
  \begin{align*}
   \Er_\A(S) := \|\A - \P_S\A\|_F^2,
  \end{align*}
  where $\|\cdot\|_F$ is the Frobenius norm, $\P_S$ is the
  projection onto $\mathrm{span}\{\a_i:i\in S\}$ and $\a_i$ denotes the $i$th
  column of~$\A$.
\end{definition}
Another variant of the CSSP emerges in the kernel
setting under the name \emph{Nystr\"om method}
\citep{Williams01Nystrom,dm_kernel_JRNL,revisiting-nystrom}. We also
discuss this variant, showing how our analysis applies in this context.
Both the CSSP and the Nystr\"om method are ways of
constructing accurate low-rank approximations by using submatrices of
the target matrix. Therefore, it is natural to ask how
close we can get to the best possible rank~$k$ approximation error:
\begin{align*}
  \opt := \!\!\min_{\B:\,\rank(\B)=k} \!\!\|\A-\B\|_F^2\leq \min_{S:\,|S|=k}\Er_\A(S).
\end{align*}
Our goal is to find a subset $S$ of size $k$ for which the ratio
between $\Er_\A(S)$ and $\opt$ is small. Furthermore, a brute force
search requires iterating over all 
${n \choose k}$ subsets, which is prohibitively expensive, so we
would like to find our subset more efficiently.

\ifisarxiv\begin{figure}\else
\begin{wrapfigure}{r}{0.45\textwidth}
  \vspace{-5mm}
  \fi
  \centering  
  \includegraphics[width=\ifisarxiv 0.6\else 0.47\fi\textwidth]{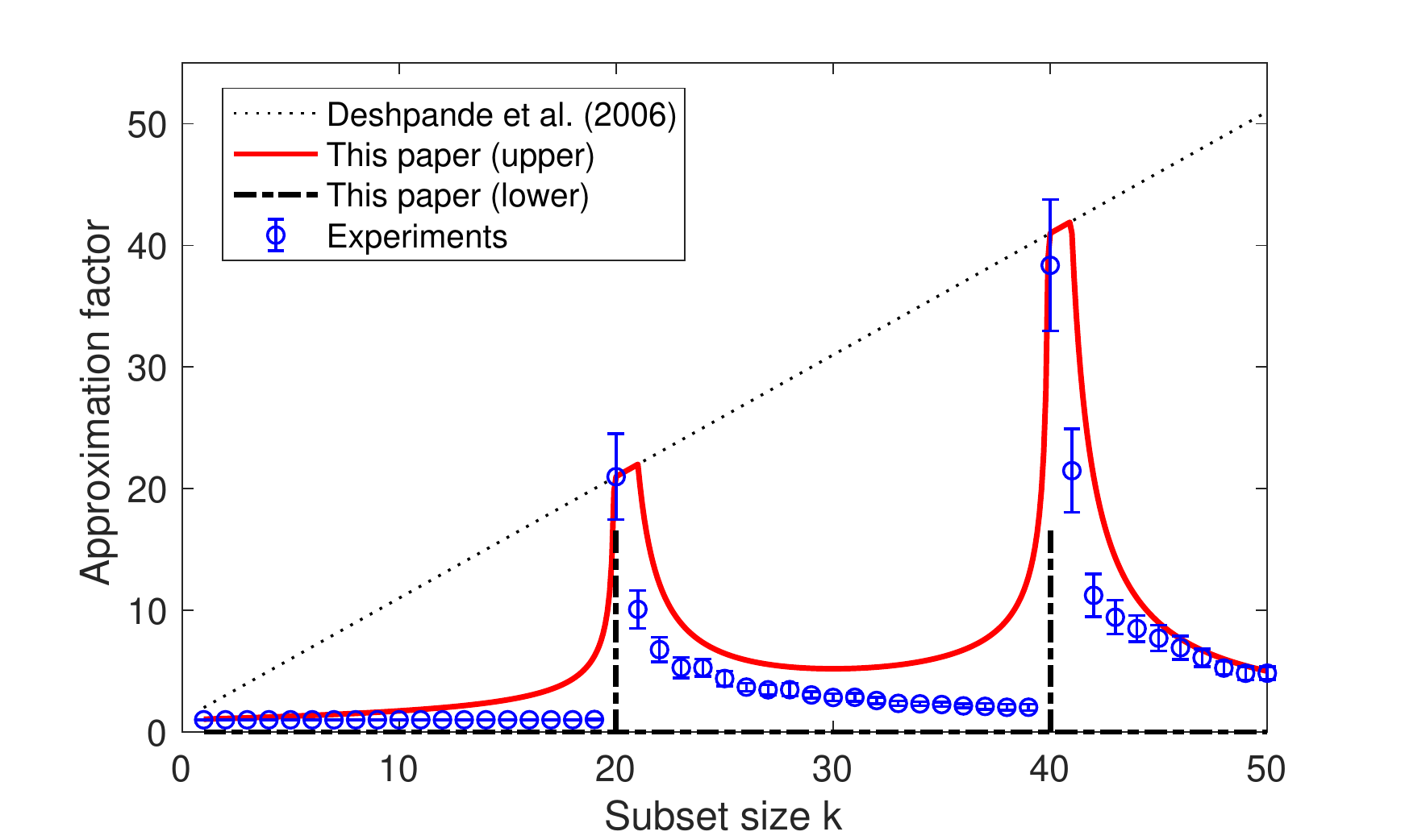}
  \caption{Empirical study of the
    expected approximation factor $\E[\Er_\A(S)]/\opt$ for
a $k$-DPP with
    different subset sizes $|S|=k$, compared to our theory. We use a
    data matrix $\A$ 
    whose
    spectrum exhibits two sharp drops, demonstrating
multiple-descent. The lower bounds are based on Theorem
    \ref{t:lower}, whereas, as our upper bound, we plot the
    minimum over all $\Phi_s(k)$ from
    Theorem~\ref{t:upper}. Note that multiple-descent vanishes under
    smooth spectral decay, resulting in improved guarantees (see
    Theorem \ref{t:decay} and Figure \ref{f:sliding}).
  }
  \label{f:intro}
  \ifisarxiv\end{figure}\else
  \end{wrapfigure}\fi

Extensive literature has been dedicated to developing algorithms for
the CSSP (an in depth discussion of the related work can be found in
Appendix \ref{s:related-work}).
In terms of worst-case analysis, 
\citet{pca-volume-sampling} gave a randomized
method 
which returns a
set $S$ of size $k$ such that: %\vspace{-4mm}
\begin{align}
  \frac{\E[\Er_\A(S)]}{\opt}\leq k+1.\label{eq:old-bound}
\end{align}
While the original algorithm was slow, 
efficient implementations have been provided since then
\citep[e.g., see][]{efficient-volume-sampling,dpp-intermediate}.
The method belongs to the family of cardinality constrained
Determinantal Point Processes \citep[DPPs, see][]{dpps-in-randnla}, and will
be denoted as $S\sim k\text{-}\DPP(\A^\top\A)$. 
The approximation  factor  $k+1$ is optimal in the worst-case,
since for any $0<k<n\leq m$ and $0<\delta<1$, an $m\times n$
matrix~$\A$ can be constructed for which $\frac{\Er_{\A}(S)}{\opt}\geq
(1-\delta)(k+1)$ for all subsets $S$ of size $k$. 
Yet it is known that, in practice, CSSP algorithms perform better
than worst-case, so the question we consider is: how can we go beyond
the usual worst-case analysis to accurately reflect what is possible
in the CSSP?  

\textbf{Contributions.} We provide improved guarantees for the CSSP
approximation factor, which go beyond the worst-case analysis and 
which lead to surprising conclusions.%\vspace{-3mm}
\begin{enumerate}
\item\underline{New upper bounds}:
We develop a family of upper bounds on the CSSP approximation factor
(Theorem~\ref{t:upper}), which we call the Master Theorem
as they can be used to derive a number 
of new guarantees. In particular, we show that when the data matrix $\A$ exhibits a
known spectral decay, then \eqref{eq:old-bound} can often be
drastically improved (Theorem \ref{t:decay}).
\item \underline{New lower bound}:
Even though the worst-case upper bound in \eqref{eq:old-bound} can often
be loose, there are cases when it cannot be improved.
We give a new lower bound construction (Theorem \ref{t:lower}) showing
that there are matrices $\A$ for which multiple different subset sizes exhibit
worst-case behavior.
\item \underline{Multiple-descent curve}: 
Our upper and lower bounds
reveal that for some matrices the CSSP approximation factor can exhibit
peaks and valleys as a function of the subset size $k$ (see Figure
\ref{f:intro}).  We show that this phenomenon is an inherent property
of the CSSP  (Corollary \ref{c:multiple-descent}).
\end{enumerate}

\subsection{Main results}
Our upper bounds rely on the notion of effective dimensionality called
stable rank~\citep{ridge-leverage-scores}. Here, we use an extended version of this concept, as
defined by \citet{BLLT19_TR}.
\begin{definition}[Stable rank]
Let $\lambda_1\geq \lambda_2\geq ...$ denote the eigenvalues of the
matrix $\A^\top\A$. For  $0\leq s<\rank(\A)$, we define the stable
rank of order $s$ as $\sr_s(\A) = \lambda_{s+1}^{-1}\sum_{i>s}\lambda_i$.
\end{definition}
In the following result, we define a family of functions $\Phi_s(k)$ which
bound the approximation factor $\Er_\A(S)/\opt$ in the range of $k$ between $s$ and
$s+\sr_s(\A)$. We call this the Master Theorem because we use it to derive
a number of more specific upper bounds.
\begin{theorem}[Master Theorem]\label{t:upper}
Given $0\leq s<\rank(\A) $, let  $t_s = s+\sr_s(\A)$,
and suppose that $s+ \frac7{\epsilon^4}\ln^2\!\frac1\epsilon \leq k\leq t_s-1$,
where $0<\epsilon\leq\frac12$. If $S\sim k$-$\DPP(\A^\top\A)$, then
\begin{align*}
  \frac{\E[\Er_\A(S)]}{\opt}\leq (1+2\epsilon)^2\,\Phi_{s}(k),
  \qquad\text{where}\quad \Phi_s(k)=\big(1+\tfrac{s}{k-s}\big)\sqrt{1 + \tfrac{2(k-s)}{t_s-k}\,}.
\end{align*}
\end{theorem}
Note that we separated out the dependence on $\epsilon$ from the function
$\Phi_s(k)$, because the term $(1+2\epsilon)^2$ is an artifact of a
concentration of measure analysis that is unlikely to be of practical
significance. In fact, we believe that the dependence on $\epsilon$
can be eliminated from the statement entirely (see
Conjecture \ref{c:convex}). 

We next examine the consequences of
the Master Theorem, starting with a sharp transition that occurs as $k$
approaches the stable rank of $\A$. 
\begin{remark}[Sharp transition]\label{r:phase-transition}
For any $k$ it is true that:
%  \vspace{-3mm}
  \begin{enumerate}
\item  For all $\A$, if $k\leq \sr_0(\A)\!-\!1$, then there is a subset $S$ of size $k$ such that
$\frac{\Er_\A(S)}{\opt} = O(\sqrt  k\,)$.
\item There is $\A$ such that $\sr_0(\A)\!-\!1<k<\sr_0(\A)$ and for
  every size $k$ subset $S$, $\frac{\Er_\A(S)}{\opt}
  \geq 0.9\,k$.
  \end{enumerate}
\end{remark}
Part 1 of the remark follows from the Master Theorem by setting
$s=0$, whereas part 2 follows from the lower bound of
\citet{more-efficient-volume-sampling}. Observe how the worst-case
approximation factor jumps from $O(\sqrt k\,)$ to $\Omega(k)$, as $k$
approaches $\sr_0(\A)$. An example of this sharp
transition is shown in Figure \ref{f:intro}, where the stable
rank of $\A$ is around $20$. 

While certain matrices directly exhibit the sharp transition from
Remark \ref{r:phase-transition}, many do not. In particular, for
matrices with a known rate of spectral decay, the Master Theorem can
be used to provide improved guarantees on the CSSP approximation
factor over \emph{all} subset sizes. 

To illustrate this, we give novel bounds
for the two most commonly studied decay rates:  polynomial and exponential.
\begin{theorem}[Examples without sharp transition]\label{t:decay}
Let $\lambda_1\!\geq\!\lambda_2\!\geq\!...$ be the
eigenvalues of $\A^\top\A$. There is an absolute constant $c$ such that
for any $0\!<\!c_1\!\leq\!c_2$, with $\gamma=c_2/c_1$, if:\\[2mm]
\textnormal{1.} (\textbf{polynomial spectral decay}) $c_1i^{-p}\!\leq\!
  \lambda_i\!\leq\! c_2i^{-p}$ $\forall_i$, with $p>1$, then $S\sim
k$-$\DPP(\A^\top\A)$ satisfies
  \begin{align*}
    \frac{\E[\Er_\A(S)]}{\opt}\leq c \gamma p. 
  \end{align*}
\textnormal{2.} (\textbf{exponential spectral decay}) $c_1(1\!-\!\delta)^{i}\leq \lambda_i\leq
  c_2(1\!-\!\delta)^{i}$ $\forall_i$, $\delta\in(0,1)$, then $S\sim
k$-$\DPP(\A^\top\A)$ satisfies %\vspace{-4mm}
  \begin{align*}
    \frac{\E[\Er_\A(S)]}{\opt}\leq c\gamma(1+ \delta k).
  \end{align*}
\end{theorem}
Note that for polynomial decay, unlike in \eqref{eq:old-bound}, the
approximation factor is constant, i.e., it does not depend on $k$. For
exponential decay, our bound provides an improvement over
\eqref{eq:old-bound} when $\delta=o(1)$.
To illustrate how these types of bounds can be obtained from the
Master Theorem, consider the function $\Phi_s(k)$ for some
$s>0$. The first term in the function, $1+\frac{s}{k-s}$, decreases with $k$, whereas
the second term (the square root) increases, albeit at a slower rate. This creates a
U-shaped curve which, if sufficiently wide, has a valley where the
approximation factor can get arbitrarily close to 1. This will occur
when $\sr_s(\A)$ is large, i.e., when the spectrum of $\A^\top\A$ has
a relatively flat region after the $s$th eigenvalue (Figure \ref{f:intro}
for $k$ between 20 and 40). Note that a peak value of some function
$\Phi_{s_1}$ may coincide with a valley of some $\Phi_{s_2}$, so only
taking a minimum over all functions reveals the true approximation
landscape predicted by the Master Theorem. To prove Theorem
\ref{t:decay}, we show that the stable ranks $\sr_s(\A)$ are
sufficiently large so that any $k$ lies in  the valley of some
function $\Phi_s(k)$ (see Section \ref{s:upper}).

The peaks and valleys of the CSSP approximation
factor suggested by Theorem \ref{t:upper} are in fact an
inherent property of the problem, rather than an artifact of our
analysis or the result of using a particular algorithm. We prove this by
constructing a family of matrices $\A$ for which the best possible approximation
factor is large, i.e., close to the worst-case upper bound of
\citet{pca-volume-sampling}, not just for one size $k$, but for a
sequence of increasing sizes.
 \begin{theorem}[Lower bound]\label{t:lower}
For any $\delta\in(0,1)$ and
$0\!=\!k_0\!<\!k_1\!<\!...\!<\!k_t\!<\!n\leq m$, there is a matrix 
$\A\in\R^{m\times n}$ such that for any subset $S$ of size $k_i$,
where $i\in\{1,...,t\}$,
\begin{align*}
\frac{\Er_{\A}(S)}{\textsc{OPT}_{k_i}}\geq (1-\delta)(k_i-k_{i-1}).
  \end{align*}
\end{theorem}
Combining the Master Theorem with the lower
bound of Theorem \ref{t:lower} we can easily provide an example matrix
for which the optimal solution to the CSSP problem exhibits multiple
peaks and valleys. We refer to this phenomenon as the multiple-descent curve.
\begin{corollary}[Multiple-descent curve]\label{c:multiple-descent}
  For $t\in\N$ and $\delta\in(0,1)$, there is a
  sequence $0<k_1^l<k_1^u<k_2^l<k_2^u<...<k_t^l<k_t^u$ and
  $\A\in\R^{m\times n}$ such that for any $i\in\{1,...,t\}$:
  \begin{align*}
    \min_{S:|S|=k_i^l}\frac{\Er_\A(S)}{\textsc{OPT}_{k_i^l}}\leq 1+\delta
                                 \qquad\text{and }\qquad
    \min_{S:|S|=k_i^u}\frac{\Er_\A(S)}{\textsc{OPT}_{k_i^u}} \geq
    (1-\delta)(k_i^u+1).
  \end{align*}
\end{corollary}
\paragraph{Connection to double descent.}
A number of phase transitions have been recently observed in the
machine learning literature which are commonly dubbed \emph{double
  descent}. The term was introduced by \citet{BHMM19} in the context of generalization
error of statistical learning, however
\citet{double-descent-condition} observed that the behavior of the
generalization error, at least for linear models, can be explained by
a more fundamental double descent phenomenon observed in the condition
number of random matrices. Also, \citet{zhenyu2020double} showed that
double descent is merely a manifestation of the
spectral phase transitions in high-dimensional random kernel
matrices. Further, \citet{liang2020multiple} observed that the spectral phase transitions
of certain kernels can lead to multiple double descent peaks. The
multiple-descent phenomenon in the CSSP that 
emerges from our analysis is
related to those works in that the spectral properties of the data
matrix (and, in particular, the condition number) determine the peaks
(i.e., phase transitions) discussed in
Corollary~\ref{c:multiple-descent}. Those parallels manifest
themselves most clearly when comparing our analysis to the work of
\citet{BLLT19_TR}, which uses the same notion of stable rank as we do,
and \citet{surrogate-design}, where determinantal sampling plays a
central role in the analysis of double descent. However, we stress that there are
  also important  differences in our setting: (1) the CSSP is a
  \emph{deterministic} optimization task, and (2) we study the
  \emph{approximation factor}, rather than the generalization error (further discussion in
  Appendix~\ref{s:related-work}).

  \subsection{The Nystr\"om method}\label{s:nystrom}
  We briefly discuss how our results translate to guarantees for the Nystr\"om
  mehod, a variant of the CSSP in the kernel setting which has
  gained considerable interest in the machine learning literature~\citep{dm_kernel_JRNL,revisiting-nystrom}.
In this context, rather than being given the column vectors
explicitly, we consider the 
$n\times n$ matrix $\K$ whose entry $(i,j)$ is the dot product
between the $i$th and $j$th vector in the kernel space,
$\langle\a_i,\a_j\rangle_\text{K}$. A Nystr\"om approximation of $\K$ based on subset $S$
is defined as $\Kbh(S) = \C\B^{\dagger}\C^\top$, where $\B$ is the
$|S|\times |S|$ submatrix of $\K$ indexed by $S$, whereas $\C$ is the
$n\times |S|$  submatrix with columns indexed by $S$. The Nystr\"om
method has many applications in machine learning, including for
kernel machines \citep{Williams01Nystrom}, Gaussian Process regression 
\citep{sparse-variational-gp} and Independent Component Analysis
\citep{Bach2003}. 
\begin{remark}
  If $\K=\A^\top\A$ and $\|\cdot\|_*$ is the trace norm, then
    $\big\|\K-\Kbh(S)\big\|_* = \Er_{\A}(S)$ for all $S\subseteq\{1,...,n\}$.
Moreover, the trace norm error of the best rank $k$ approximation of
$\K$, 
is equal to the squared Frobenius norm error of the 
best rank $k$ approximation of $\A$, i.e.,
\[\min_{\Kbh:\,\rank(\K)=k}\|\K-\Kbh\|_* = \opt.\]
\end{remark}
%\vspace{-2mm}
This connection was used by~\citet{belabbas-wolfe09} to adapt the
$k+1$ approximation factor bound 
of~\citet{pca-volume-sampling} to the Nystr\"om method.   
Similarly, all of our results for the CSSP, including the
multiple-descent curve that we have observed, can be translated into analogous
statements for the trace norm approximation error in the Nystr\"om
method. Of particular interest 
are the improved bounds for kernel matrices with known eigenvalue
decay rates. Such matrices arise naturally in machine learning when
using standard kernel functions such as the Gaussian Radial Basis Function
(RBF) kernel (a.k.a.~the squared exponential kernel) and the
Mat\'ern kernel \citep{sparse-variational-gp}.

\underline{RBF kernel}: If
$\langle\a_i,\a_j\rangle_\text{K}=\exp(-\|\a_i\!-\!\a_j\|^2/\sigma^2)$ 
and the data comes from $\Nc(0,\eta^2)$, then,
  for large $n$,
  $\lambda_i\!\asymp\!\lambda_1(\frac{b}{a+b+c})^i$, where
  $a=\frac1{4\eta^2}$, $b=\frac1{\sigma^2}$ and $c=\sqrt{a^2\!+\!2ab}$
  \citep{Santa97gaussianregression},
  so Theorem~\ref{t:decay} yields an approximation factor of
  $O(1\!+\!\frac{a+c}{a+b+c}k)$, better than $k\!+\!1$ when
  $\sigma^2\ll\eta^2$. Note that the parameter $\sigma$ defines the
  size of a neighborhood around which the data points are deemed
  similar by the RBF kernel. Therefore, smaller $\sigma$ means that
  each data point has fewer similar neighbors.

  \underline{Mat\'ern kernel}: If $\K$ is the Mat\'ern kernel with
  parameters $\nu$ and $\ell$ and the data is distributed according to a
  uniform measure in one dimension,
  then $\lambda_i\asymp \lambda_1 i^{-2\nu-1}$
  \citep{RasmussenWilliams06}, so Theorem~\ref{t:decay} yields a
  Nystr\"om approximation factor of $O(1+\nu)$ for any subset size
  $k$.

In Section \ref{s:experiments}, we also empirically
demonstrate our improved guarantees and the multiple-descent curve for the Nystr\"om method with the
RBF kernel.

\section{Upper bounds}\label{s:upper}\vspace{-2mm}
In this section, we derive the upper bound given in
Theorem~\ref{t:upper} by using a novel expectation formula for the squared
projection error of a DPP. We then show
how this result can be used to 
obtain improved guarantees for matrices with known eigenvalue decays,
i.e., Theorem \ref{t:decay}. Our analysis heavily relies on the theory
of DPPs \citep{dpps-in-randnla}, so for completeness, in Appendix~\ref{s:dpp} we provide a
brief summary of DPPs and the relevant results.  

Let $S\sim \DPP(\frac1\alpha\A^\top\A)$ denote a
distribution over all subsets $S\subseteq [n]$ so that
$\Pr(S)\propto\det(\frac1\alpha\A_S^\top\A_S)$, where $\alpha>0$. Then,
$k$-$\DPP(\A^\top\A)$ is simply a restriction of
$\DPP(\frac1\alpha\A^\top\A)$ to the subsets of size $k$ (regardless
of the choice of $\alpha$).
However, the expected subset size for $\DPP(\frac1\alpha\A^\top\A)$ does
depend on~$\alpha$. Our analysis relies on a careful selection of
this parameter. In Lemma \ref{l:expected-error}  (Appendix \ref{s:dpp}), we show the
following expectation formula for the CSSP approximation error: 
\[\E[\Er_\A(S)]=\E[|S|]\cdot\alpha,\quad\text{for}\quad S\sim \DPP(\tfrac1\alpha\A^\top\A).\]
If we set $\alpha =\opt = \sum_{i=k+1}^n\lambda_i$,
where $\lambda_i$ are the eigenvalues of
$\A^\top\A$ in decreasing order, then:
\begin{align*}
 \E[|S|]& = \sum_{i=1}^n\frac{\lambda_i}{\alpha+\lambda_i}
\leq\sum_{i=1}^k\frac{\lambda_i}{\alpha+\lambda_i} + 1\leq k+1.
\end{align*}
This recovers the upper bound of \citet{pca-volume-sampling}, i.e., $\E[\Er_\A(S)]\leq(k+1)\opt$,
except that the
subset size is randomized with expectation bounded by 
$k+1$, instead of a fixed subset size equal $k$. However, a more refined
choice of the parameter $\alpha$
allows us to significantly improve on the above error bound in certain
regimes, as shown below.
\begin{lemma}\label{l:upper}
For any $\A$, $0\leq\epsilon<1$ and $s<k<t_s$, where $t_s=s+\sr_s(\A)$,
say $S\sim\DPP(\frac1\alpha\A^\top\A)$ for
$\alpha=\frac{\gamma_s(k)\opt}{(1-\epsilon)(k-s)}$ and
$\gamma_s(k):=\sqrt{1+\frac{2(k-s)}{t_s-k}\,}$. Then, defining $\Phi_s(k)
:=\big(1+\frac{s}{k-s}\big)\,\gamma_s(k)$,
\begin{align*}
  \frac{\E\big[\Er_\A(S)\big]}{\opt}\leq \frac{\Phi_s(k)}{1-\epsilon}
  \quad\text{and}\quad \E[|S|]\leq k-\epsilon\,\frac{k-s}{\gamma_s(k)}.
\end{align*}
  \end{lemma}
Note that, setting $\epsilon=0$, the above lemma implies that we can
achieve approximation factor $\Phi_s(k)$ with a DPP whose expected
size is bounded by $k$. We introduce $\epsilon$ so that we can
convert the bound from DPP to the fixed size k-DPP via a
concentration argument. Intuitively, our strategy is to show that the
randomized subset size of a DPP is sufficiently concentrated around
its expectation that with high probability it will be bounded by $k$,
and for this we need the expectation to be strictly below $k$. A
careful application of the Chernoff bound for a Poisson binomial
random variable yields the following concentration bound.
\begin{lemma}\label{l:concentration}
Let $S$ be as in Lemma \ref{l:upper} with $\epsilon\leq \frac12$.
If $s+\frac7{\epsilon^4}\ln^2\!\frac1\epsilon \leq k\leq t_s-1$, then
$\Pr(|S|>k) \leq \epsilon$.
\end{lemma}
Finally, any expected bound for random size DPPs can be converted to
an expected bound for a fixed size k-DPP via the following
result.
\begin{lemma}\label{l:monotonic}
  For any $\A\in\R^{m\times n}$, $k\in[n]$ and $\alpha>0$, if
  $S\sim\DPP(\frac1\alpha\A^\top\A)$ and $S'\sim k$-$\DPP(\A^\top\A)$,
  then %\vspace{-4mm}
  \begin{align*}
    \E\big[\Er_\A(S')\big] \leq \E\big[\Er_\A(S)\mid |S|\leq k\big].
  \end{align*}
\end{lemma}
The above inequality may seem intuitively obvious since adding more columns
to a set $S$ to complete it to size $k$ always reduces the
error. However, a priori, it could happen that going from subsets 
of size $k-1$ to subsets of size $k$ results in a redistribution of
probabilities to the subsets with larger error. To show that this will
not happen, our proof relies on classic but non-trivial combinatorial
bounds called Newton's inequalities. Putting together
Lemmas~\ref{l:upper}, \ref{l:concentration} and
\ref{l:monotonic}, we obtain our Master Theorem.
\begin{proofof}{Theorem}{\ref{t:upper}}
Let $S\sim\DPP(\frac1\alpha\A^\top\A)$ be sampled as in Lemma
\ref{l:upper}, and let $S'\sim k$-$\DPP(\A^\top\A)$. We have:\vspace{-2mm}
\begin{align*}
  \E\big[\Er_\A(S')\big]
\overset{(a)}{\leq} \E\big[\Er_\A(S)\mid|S|\leq k\big]
  \leq \frac{\E\big[\Er_\A(S)\big]}{\Pr(|S|\leq k)}
  \overset{(b)}{\leq} \frac{\Phi_s(k)}{(1-\epsilon)^2}\cdot\opt,
\end{align*}
where $(a)$ follows from Lemma \ref{l:monotonic} and $(b)$ follows
from Lemmas \ref{l:upper} and \ref{l:concentration}. Since
$0<\epsilon\leq\frac12$, we have
$\frac1{(1-\epsilon)^2}\leq(1+2\epsilon)^2$, which completes the proof.
\end{proofof}

\begin{figure}
  \centering
  \ifisarxiv
  \includegraphics[width=0.6\textwidth]{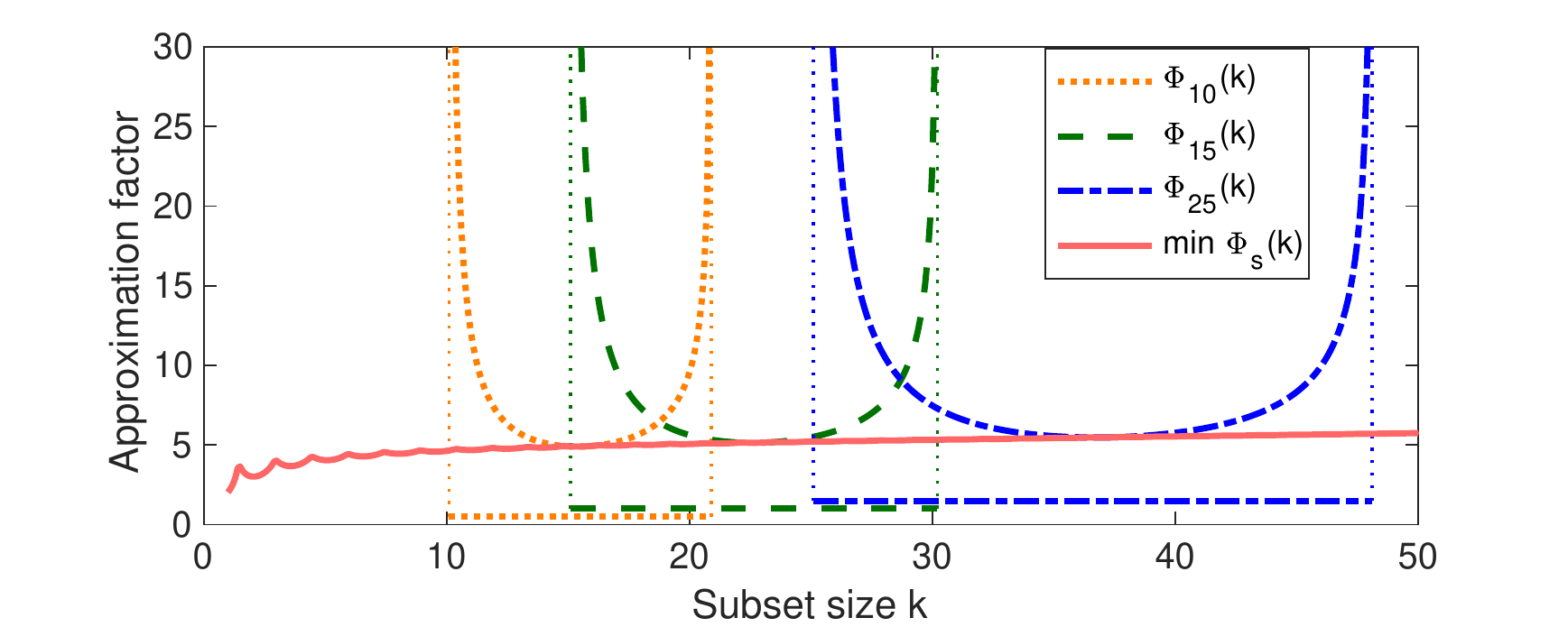}
\includegraphics[width=0.6\textwidth]{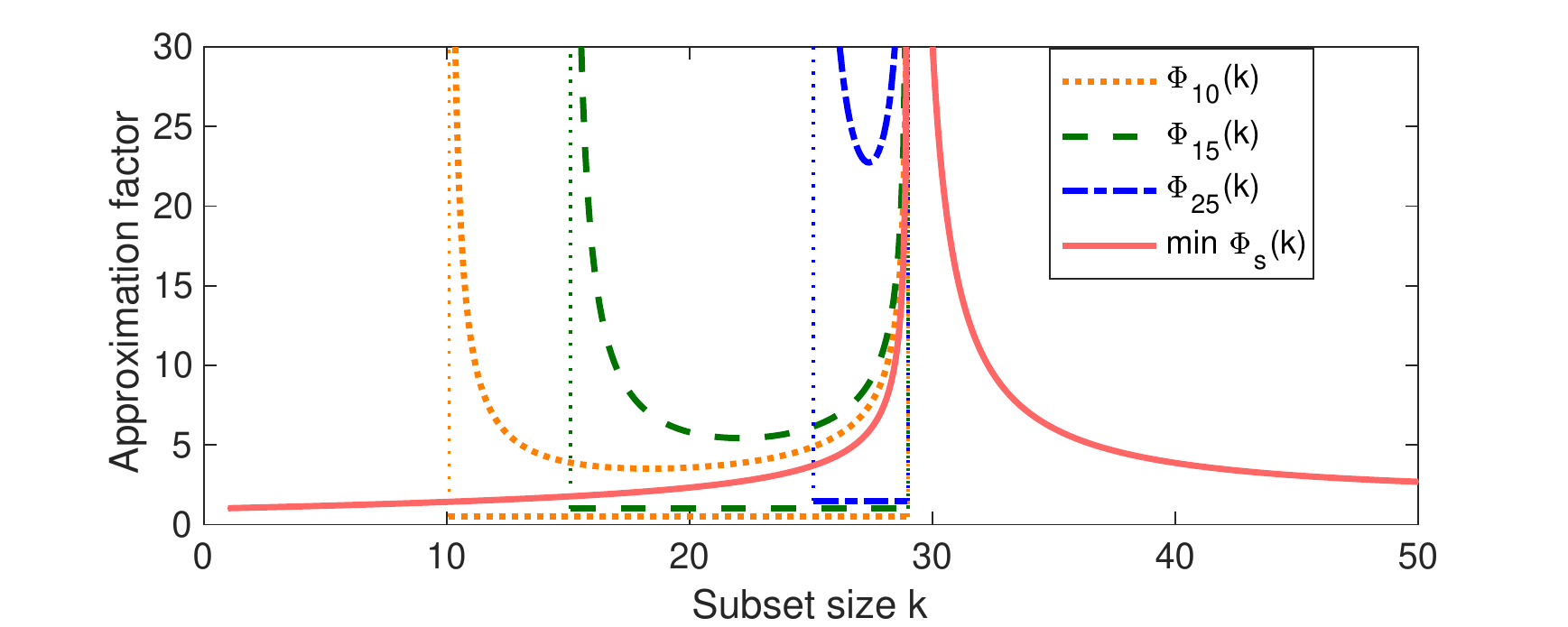}
\else  
%\centering
\hspace{-8mm}
\includegraphics[width=0.54\textwidth]{figs/nystrom/sliding}\hspace{-5mm}
\includegraphics[width=0.54\textwidth]{figs/nystrom/sliding-peak}
%  \vspace{-6mm}
  \fi
  \caption{Illustration of the upper bound functions $\Phi_s(k)$ for
    different values of $s$, with a $200\times 200$ matrix $\A$ such
    that the $i$th eigenvalue of $\A^\top\A$ is set to:
    \ifisarxiv(top)\else(left)\fi\
    $1/i$;
    \ifisarxiv(bottom)\else(right)\fi\
    $1$ for $i<30$ and $0.01$ for $i\geq 30$. For each
    function, we marked the window of applicable $k$'s with a
    horizontal line. For polynomial spectral decay
    \ifisarxiv(top)\else(left)\fi
    , the stable rank
    $\sr_s(\A)$ (i.e., the width of the window starting at $s$)
    increases, while for the sharp spectrum drop
    \ifisarxiv(bottom)\else(right)\fi\
    the stable
    rank shrinks as the window approaches the drop, causing a peak in
    the upper~bound.}\vspace{-2mm}
  \label{f:sliding}
\end{figure}

We now demonstrate how Theorem \ref{t:upper} can be used as the Master Theorem to derive new
bounds on the CSSP approximation factor under additional assumptions
on the singular value decay of matrix $\A$. Rather than a single upper
bound, Theorem \ref{t:upper} provides a family of upper bounds
$\Phi_s$, each with a range of applicable values $k$. Since each
$\Phi_s(k)$ forms a U-shaped curve, its smallest point falls near the
middle of that range. In Figure \ref{f:sliding} we visualize these bounds as a sliding
window that sweeps across the axis representing possible subset
sizes. The width of the window varies: when it starts at $s$ 
then its width is the stable rank $\sr_s(\A)$. The wider the window,
the lower is the valley of the corresponding U-curve. Thus, when
bounding the approximation factor for a given $k$, we should choose
the widest window such that $k$ falls near the bottom of its
U-curve. Showing a guarantee that holds for all $k$ requires lower-bounding the
stable ranks $\sr_s(\A)$ for each $s$. This is straightforward for
both polynomial and exponential decay. Specifically, using the
notation from Theorem \ref{t:decay}, in Appendix \ref{a:decay} we
prove that: %\vspace{-3mm}
\begin{align*}
  \sr_s(\A) =
  \begin{cases}
    \Omega(s/p), &\text{for polynomial rate  $\lambda_i\asymp1/i^p$},\\
    \Omega(1/\delta),&\text{for exponential rate $\lambda_i\asymp(1-\delta)^i$}.
    \end{cases}
  \end{align*}
As an example, Figure \ref{f:sliding} (left) shows that the stable rank $\sr_s(\A)$, i.e.,
the width of the window starting at $s$, grows linearly with $s$ for eigenvalues
decaying polynomially with $p=1$. As a result, the bottom of each
U-shaped curve remains at roughly the same level, making the CSSP
approximation factor independent of $k$, as in Theorem
\ref{t:decay}. In contrast, Figure \ref{f:sliding} (right) provides
the same plot for a different matrix $\A$ with a sharp drop in the spectrum. The U-shaped
curves cannot slide smoothly across that drop because of the
shrinking stable ranks, which results in a peak similar to the ones
observed in Figure \ref{f:intro}.

 \section{Lower bound}\label{s:lower}
%\vspace{-2mm}
 As discussed in the previous section, our upper bounds for the CSSP
 approximation factor exhibit a peak (a high point,
 with the bound decreasing on either side) around a
 subset size $k$ when there is a sharp drop in the spectrum of $\A$
 around the $k$th singular value. It is natural to ask whether this
 peak is an artifact of our analysis, or a property of the k-DPP
 distribution, or whether even optimal CSSP subsets exhibit this 
 phenomenon. In this section, we extend a lower
 bound construction of \citet{pca-volume-sampling} and use it to show
 that for certain matrices the approximation factor of the optimal CSSP subset, i.e.,
 $\min_{|S|=k}\Er_\A(S)/\opt$, can exhibit not just one but any
 number of peaks as a function of $k$, showing that the
 multiple-descent curve in Figure \ref{f:intro} describes a
real phenomenon in the CSSP.

The lower bound construction of \citet{pca-volume-sampling} relies on
arranging the column vectors of a $(k+1)\times (k+1)$ matrix $\A$ into a
centered symmetric $k$-dimensional simplex. This way, the $k+1$
columns are spanning a $k$ dimensional subspace which contains the $k$
leading singular vectors of $\A$. They then proceed to shift the
columns slightly in the direction orthogonal to that subspace so that
the $(k+1)$st singular value of $\A$ becomes non-zero. This results in
an instance of the CSSP with a sharp drop in the spectrum. Due
to the symmetry in this construction, all subsets of size $k$ have an identical squared
projection error. It is easy to show that this error satisfies $\Er_\A(S)\geq (1-\delta)(k+1)\opt$,
where $\delta$ is a parameter which depends on the condition number of
matrix $\A$ and it can be driven arbitrarily close to $0$. Another
variant of this construction was also provided 
by \citet{more-efficient-volume-sampling}. The key limitation of both
of these constructions is that they only provide a lower bound for a
single subset size $k$ in a given matrix, whereas our goal is to show
that the CSSP can exhibit the multiple-descent curve, which requires
lower bounds for multiple different values of $k$ holding with respect
to the same matrix $\A$.

Our strategy for constructing the lower bound matrix is to concatenate
together multiple sets of columns, each of which
represents a simplex spanning some subspace of $\R^m$. The key
challenge that we face in this approach is that, unlike in
the construction of \citet{pca-volume-sampling}, different subsets of
the same size will have different projection errors. 
\begin{lemma}\label{l:lower}
Fix $\delta\in(0,1)$ and consider unit vectors
$\a_{i,j}\in\R^m$ in general position, where $i\in[t]$, $j\in[l_i]$, such that
$\sum_j\a_{i,j}=0$ for each $i$, and for any $i,j,i',j'$, if $i\ne i'$
then $\a_{i,j}$ is orthogonal to $\a_{i'\!,j'}$. Also, let unit vectors
$\{\v_i\}_{i\in[t]}$ be orthogonal to  each other and to all
$\a_{i,j}$. There are positive scalars $\alpha_{i},\beta_i$ for $ i \in [t]$
such that matrix $\A$ with columns 
$\alpha_{i}\a_{i,j}+\beta_i\v_i$ over all $i$ and $j$ satisfies:
%\ifisarxiv\else\vspace{-1mm}\fi
\begin{align*}
  \min_{|S|=k_i}\frac{\Er_\A(S)}{\mathrm{OPT}_{k_i}}\geq
  (1-\delta)l_i,\quad\text{for }k_i=l_1+...+l_i-1.
\end{align*}
\end{lemma}
%\ifisarxiv\else \vspace{-5mm}\fi
% Theorem \ref{t:lower} now follows easily from the lemma.
\begin{proofof}{Theorem}{\ref{t:lower}}
We let $l_1 = k_1+1$ and then for $i>1$ we set $l_i=k_i-k_{i-1}$. We
then construct the vectors $\a_{i,j}$ that satisfy Lemma \ref{l:lower} 
by letting each set $\{\a_{i,j}\}_{j}$ be the corners of a centered
$(l_i-1)$-dimensional regular simplex. We ensure that each
simplex is orthogonal to every other simplex by placing them in
orthogonal subspaces. 
\end{proofof}

\section{Empirical evaluation}
\label{s:experiments}
\vspace{-2mm}

In this section, we provide an empirical evaluation designed to demonstrate how our improved guarantees for the CSSP and Nystr\"om method, as well as the multiple-descent phenomenon, can be easily observed on real datasets. 
We use a standard experimental setup for data subset selection using
the Nystr\"om method \citep{revisiting-nystrom}, where an $n\times n$
kernel matrix $\K$ for a dataset of size $n$ is defined so that the
entry $(i,j)$ is computed using the Gaussian Radial Basis Function  (RBF)
kernel:
$\langle\a_i,\a_j\rangle_\text{K}=\exp(-\|\a_i\!-\!\a_j\|^2/\sigma^2)$,
where $\sigma$ is a free parameter. 
We are particularly interested in the effect of varying $\sigma$.
Nystr\"om subset selection is performed using $S\sim k$-$\DPP(\K)$
(Definition \ref{d:dpp}), and we plot the expected approximation
factor $\E[\|\K-\Kbh(S)\|_*]/\opt$ (averaged over 1000 runs), where
$\Kbh(S)$ is the Nystr\"om approximation of $\K$ based on the subset
$S$ (see Section~\ref{s:nystrom}), $\|\cdot\|_*$ is the trace norm,
and $\opt$ is the trace norm error of the best rank $k$
approximation. Additional experiments, using greedy selection instead 
of a k-DPP, are in Appendix~\ref{a:greedy}.
As discussed in Section \ref{s:nystrom}, this task is
equivalent to the CSSP task defined on the matrix $\A$ such that
$\K=\A^\top\A$.

The aim of our empirical evaluation is to verify the following two claims motivated by our theory (and to illustrate that doing so is as easy as varying the RBF parameter $\sigma$):
%\ifisarxiv\else\vspace{-2mm}\fi
\begin{enumerate}
  \item When the spectral decay is sufficiently slow/smooth, the
    approximation factor for CSSP/Nystr\"om is much better than
    suggested by previous worst-case bounds.
%    \vspace{-1mm}
  \item  A drop in spectrum around the $k$th eigenvalue results in
    a peak in the approximation factor near
    subset size~$k$. Several drops result in the
    multiple-descent curve.
  \end{enumerate}
%\ifisarxiv\else  \vspace{-2mm}\fi
In Figure \ref{f:rbf} (top), we plot the approximation factor against
the subset size $k$ (in the range of 1 to 40) for an artificial toy
dataset and 
for two benchmark regression datasets from the Libsvm repository
\citep[\emph{bodyfat} and \emph{eunite2001}, see][]{libsvm}. 
The toy dataset is constructed by scaling the eigenvalues of a random
$50\times 50$ Gaussian matrix so that the spectrum is flat with a
single drop at the $21$-st eigenvalue. 
For each dataset, in Figure~\ref{f:rbf} (bottom), we also show the top
40 eigenvalues of the kernel $\K$ in decreasing order. 
For the toy dataset, to maintain full control over the spectrum we use
the linear kernel $\langle\a_i,\a_j\rangle_\text{K}=\a_i^\top\a_j$, and we show results
for three different values of the condition number $\kappa$ of kernel
$\K$.  
For the benchmark datasets, we show results on the RBF kernel with 
three different values of the parameter~$\sigma$.  

\begin{figure*}[t]
   \centering
  \ifisarxiv
  \includegraphics[width=0.345\textwidth]{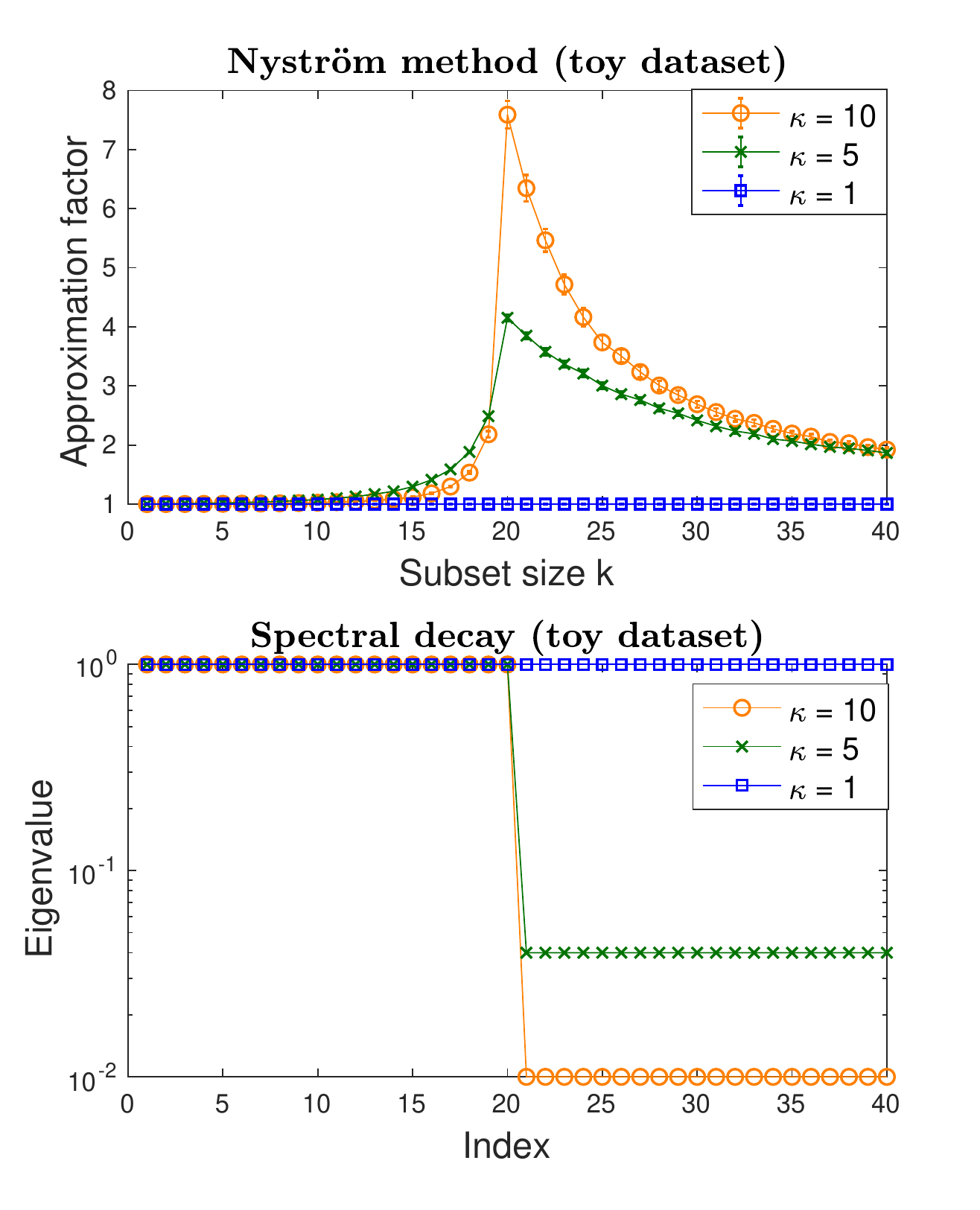}
   \hspace{-6mm}
  \includegraphics[width=0.345\textwidth]{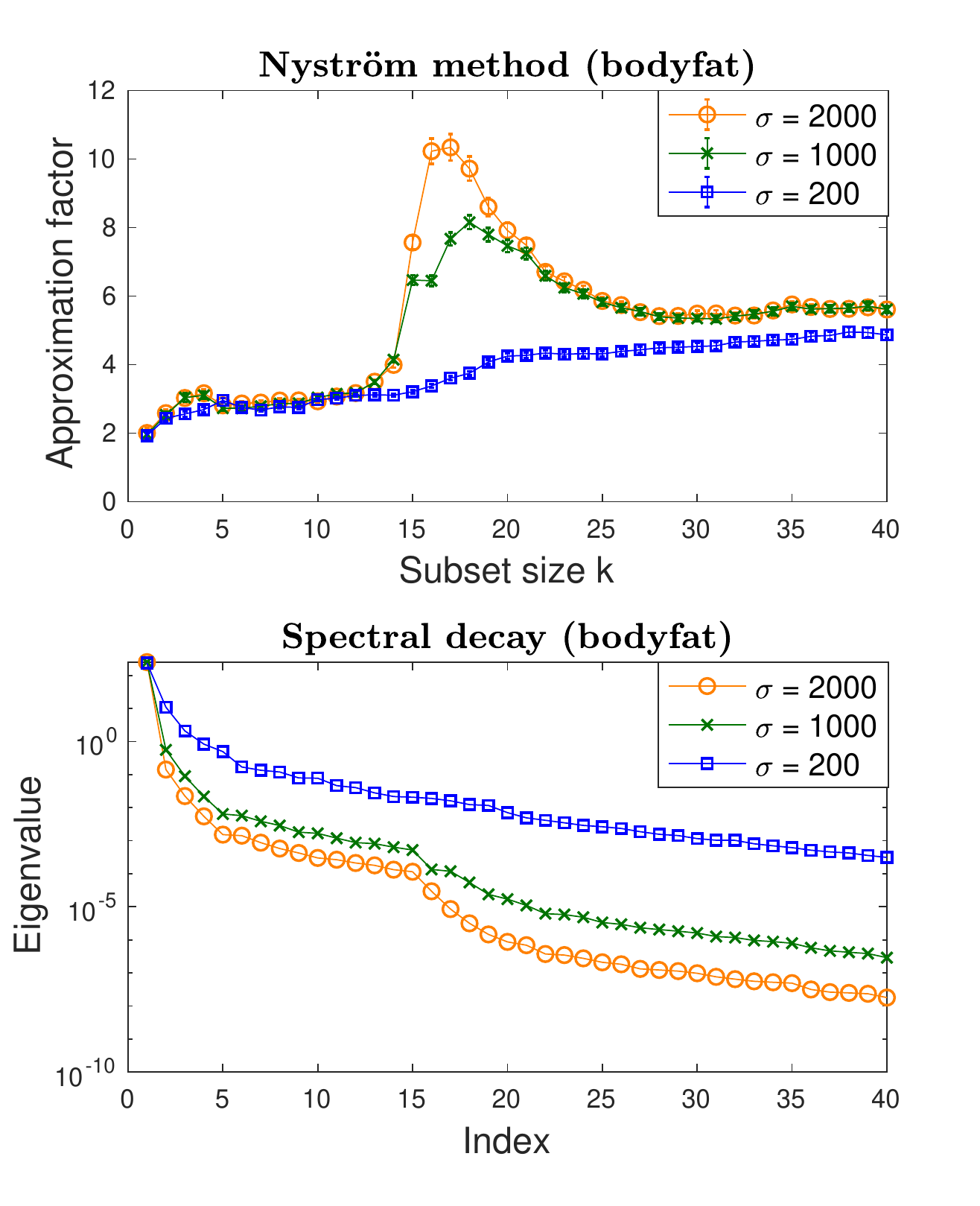}
  \hspace{-6mm}
  \includegraphics[width=0.345\textwidth]{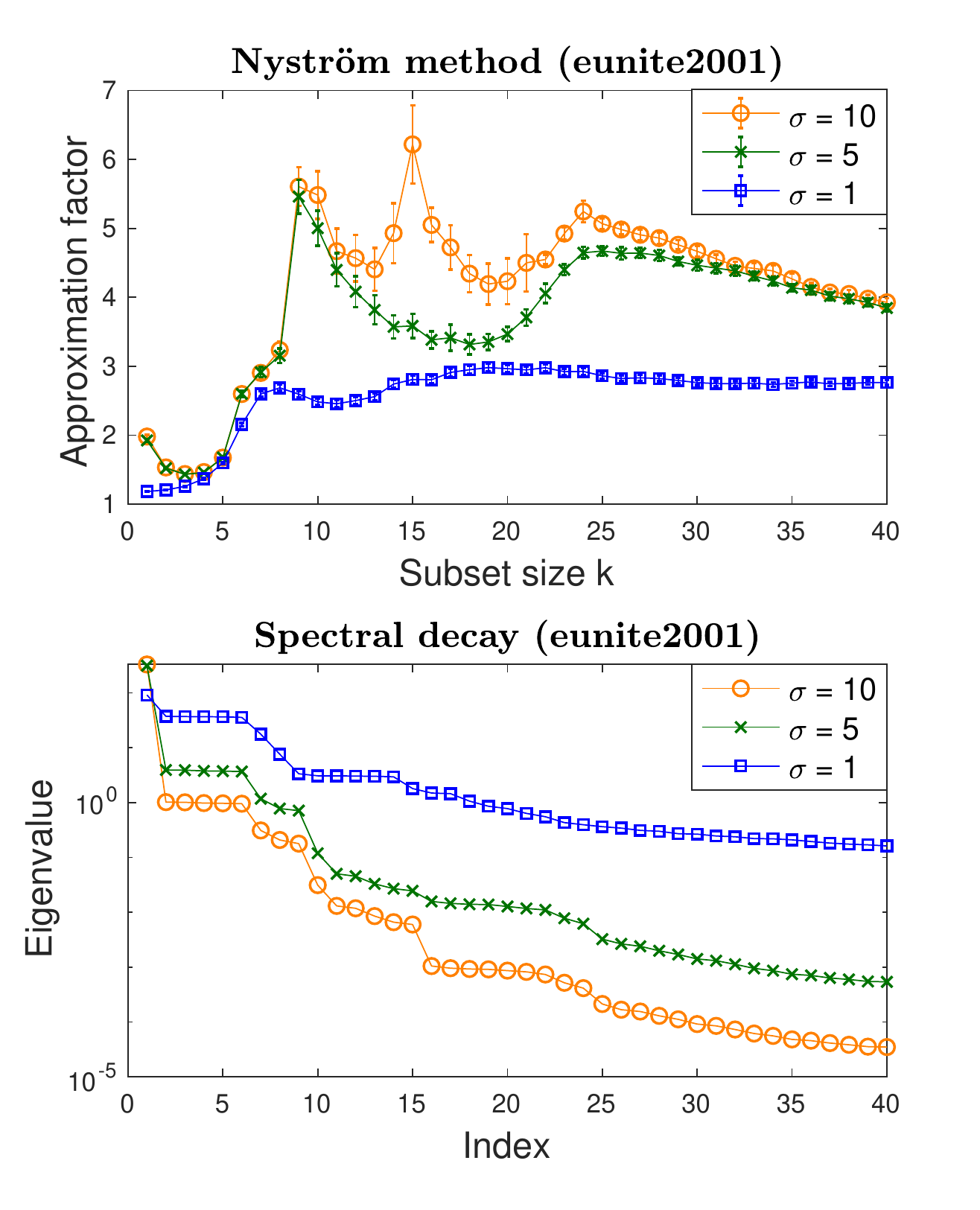}
  \vspace{-5mm}
  \else
    \includegraphics[width=0.35\textwidth]{figs/nystrom/rbf-toy-double}
   \hspace{-6mm}
  \includegraphics[width=0.35\textwidth]{figs/nystrom/rbf-bodyfat-double}
  \hspace{-6mm}
  \includegraphics[width=0.35\textwidth]{figs/nystrom/rbf-eunite2001-double}
  \vspace{-4mm}
  \fi
  \caption{Top three plots show the Nystr\"om approximation factor $\E[\|\K-\Kbh(S)\|_*]/\opt$,
    where $S\sim k$-$\DPP(\K)$ (experiments using greedy selection
    instead of a k-DPP are in Appendix \ref{a:greedy}), for a toy dataset
    ($\kappa$ is the condition number) and two
    Libsvm datasets ($\sigma$ is the RBF parameter). Error bars show three times the standard
    error  of the mean over 1000 trials. Bottom three plots show the spectral decay for  the
  top $40$ eigenvalues  of each kernel $\K$. Note that the
  peaks in the approximation factor align with the drops in
  the spectrum.}
\vspace{-2mm}
  \label{f:rbf}
\end{figure*}

Examining the toy dataset (Figure~\ref{f:rbf}, left), it is apparent
that a larger drop in spectrum leads to a sharper peak in the
approximation factor as a function of the subset size~$k$, whereas a flat spectrum results in  
the approximation factor being close to $1$. A similar trend is observed for
dataset \emph{bodyfat} (Figure~\ref{f:rbf}, center), where large
parameter $\sigma$ results in a peak that is aligned with a spectrum
drop, while 
decreasing $\sigma$ makes the spectrum flatter and the factor closer to
$1$. Finally, dataset \emph{eunite2001} (Figure~\ref{f:rbf}, right) exhibits a full
multiple-descent curve with up to three peaks for large values of
$\sigma$, and the peaks are once again aligned with the spectrum
drops. Decreasing $\sigma$ gradually eliminates the peaks, 
resulting in a uniformly small approximation factor. Thus, both of our theoretical
claims can easily be verified on this dataset simply by adjusting the RBF parameter.

While the right choice of the parameter $\sigma$ ultimately depends on
the downstream machine learning task, it has been observed that
varying $\sigma$ has a pronounced effect on the spectral properties of
the kernel matrix, \citep[see,
e.g.,][]{revisiting-nystrom,LBM16_TR,wang2019block}.  
The main takeaway from our results here is that, depending on the
structure of the problem, we may end up in the regime where the
Nystr\"om approximation factor exhibits a multiple-descent curve
(e.g., due to a hierarchical nature of the data) or in the regime
where it is relatively flat.

\section{Conclusions and open problems}
\label{s:conclusions}
%\vspace{-3mm}
We derived new guarantees for the Column Subset Selection Problem (CSSP) and
the Nystr\"om method, going beyond worst-case analysis by exploiting the
structural properties of a dataset, e.g., when the spectrum exhibits a
known rate of decay. Our upper and lower bounds
for the CSSP/Nystr\"om approximation factor reveal an intruiguing phenomenon we
call the multiple-descent curve: the approximation factor can exhibit a highly
non-monotonic behavior as a function of $k$, with multiple
peaks and valleys. These observations suggest a possible connection to the
double descent curve exhibited by
the generalization error of many machine learning models (see
Appendix~\ref{s:related-work} for a more in depth discussion of
similarities and differences between the two phenomena).

Our analysis technique relies on converting an error bound
from random-size DPPs to fixed-size k-DPPs, which results in an
additional constant factor of $(1+2\epsilon)^2$ in Theorem~\ref{t:upper}.
We put forward a
conjecture which would eliminate this factor from
Theorem~\ref{t:upper} and is of independent interest to the study of 
elementary symmetric polynomials, a classical topic in combinatorics
\citep{HLP-Inequalities}. 
\begin{conjecture}\label{c:convex}
  The following function is \underline{convex} with respect to $k\in[n]$ for any
$\lambda_1,...,\lambda_n> 0$:\vspace{-.5mm}
 \[f(k)=(k+1)\frac{\sum_{S:|S|=k+1}\prod_{i\in S}\lambda_i}{\sum_{S:|S|=k}\prod_{i\in S}\lambda_i}.\]
\end{conjecture} %\vspace{-1.5mm}

\citet{pca-volume-sampling} showed that if $S\sim
k$-$\DPP(\A^\top\A)$ and $\lambda_i$ are the eigenvalues of
$\A^\top\A$, then $\E[\Er_\A(S)]=f(k)$. If $f(k)$ is convex then
Jensen's inequality implies:
\begin{align*}
  \E[\Er_\A(S)]\leq \E[\Er_\A(S')]\quad\text{for }S'\!\sim\DPP(\tfrac1{\alpha_k}\A^\top\A),
\end{align*}
where $\alpha_k$ is chosen so that $\E[|S'|]=k$. This would allow us to use
the bound from Lemma \ref{l:upper} directly on a k-DPP
without relying on the concentration argument of
Lemma~\ref{l:concentration}, thereby improving the bounds in
Theorems \ref{t:upper} and \ref{t:decay}.

\subsection*{Acknowledgments}
We would like to acknowledge DARPA, IARPA (contract W911NF20C0035),
NSF, and ONR via its BRC on RandNLA for providing partial support of
this work.  Our conclusions do not necessarily reflect the position or
the policy of our sponsors, and no official endorsement should be
inferred.

\ifisarxiv\else
\section*{Broader Impact}
Our study offers a deeper theoretical understanding of the Column
Subset Selection Problem. The nature of our work is primarily
theoretical, with direct applications to feature selection and kernel
approximation, as we noted in Section \ref{s:intro}. The primary reason for
feature selection as a method for approximating a given matrix, as opposed to a low rank
approximation using an SVD, is interpretability, which is crucial in
many scientific disciplines. Our analysis shows that in many practical settings, feature
selection performs almost as well as SVD at approximating a matrix. As
such, our work makes a stronger case for feature selection, wherever
applicable, for the sake of interpretability. We also hope our work
motivates further research into a fine grained analysis to quantify if
machine learning problems are really as hard as worst-case bounds
suggest them to be.
\fi

\bibliographystyle{sty/icml2020/icml2020}
\bibliography{pap}

  \newpage
  \appendix
  \onecolumn

\section{Additional related works}
\label{s:related-work}
\underline{The Column Subset Selection Problem} is one of the most classical tasks
in matrix approximation~\citep{BoutsidisMD08}. The original version of the problem compares the
projection error of a subset of size $k$ to the best rank $k$
approximation error. The techniques used for finding good subsets have
included many randomized methods
\citep{pca-volume-sampling,BoutsidisMD08,dpp-for-cssp,boutsidis2014optimal}, as well as deterministic
methods \citep{GuE96}. Variants of these algorithms have also been extended to more general losses~\citep{chierichetti17algorithms, khanna2017scalable, elenberg2016restricted}. Later on, most works have relaxed the problem
formulation by allowing the number of selected columns $|S|$ to exceed
the rank $k$. These approaches include deterministic sparsification
based algorithms \citep{near-optimal-columns}, greedy selection
\citep[e.g.,][]{Bhaskara2016GreedyCSS} and randomized methods
\citep[e.g.,][]{Drineas08CUR,more-efficient-volume-sampling,Paul2015CSS}. Note that
we study the \emph{original} version of the CSSP (i.e., without the
relaxation), where the number of columns $|S|$ must be equal to the
rank $k$. 

\underline{The Nystr\"om method} has been given significant attention
independently of the CSSP. The guarantees most comparable to our
setting are due to \citet{belabbas-wolfe09}, who show the approximation
factor $k+1$ for the trace norm error. Many recent works allow the subset size $|S|$ to exceed
the target rank $k$, which enables the use of i.i.d.~sampling
techniques such as leverage scores \citep{revisiting-nystrom}
and ridge leverage scores \citep{ridge-leverage-scores,Musco17Nystrom}.
In addition to the trace norm error, these works consider other types
of guarantees, e.g., based on spectral and Frobenius norms, which are not 
as readily comparable to the CSSP error bounds.

\underline{The double descent curve} was introduced by \citet{BHMM19} to explain
the remarkable success of machine learning models which generalize
well despite having more parameters than training data. This research 
has been primarily motivated by the success of deep neural networks,
but double descent has also been observed in linear regression
\citep{belkin2019two,BLLT19_TR,surrogate-design} and other learning models.
Double descent is typically presented by plotting
the absolute generalization error as a function of the number of parameters used in the
learning model, although \citet{double-descent-condition} and
\citet{zhenyu2020double} showed that the 
behavior of generalization error is merely an artifact of the phase
transitions in the spectral properties of random
matrices. Importantly, although the descent curves we obtain are 
reminiscent of the above works, our setting is different in that it is
a \emph{deterministic} combinatorial 
optimization problem for \emph{relative} error. In particular,
Corollary~\ref{c:multiple-descent} shows that our multiple-descent
curve can 
occur as a purely deterministic property of the optimal CSSP
solution. Despite the differences, there are certain similarities
between the two settings, namely (a) the notion of stable rank we use
matches the one used by~\citet{BLLT19_TR}, (b) the peaks in both the
settings are closely aligned -- these peaks coincide with the size $k$
crossing the corresponding sharp drops in the respective spectra, (c)
the analysis of bias of the minimum norm solution for double descent
for linear regression under DPP sampling obtained
by~\citet{surrogate-design} leads to expressions very similar to ours
for the CSSP error for DPP sampling. 

\underline{Determinantal point processes} have been shown to provide
near-optimal guarantees not only for the CSSP but also other tasks in
numerical linear algebra, such as least squares regression
\citep[e.g.,][]{avron-boutsidis13,unbiased-estimates-journal,minimax-experimental-design}.
They are also used in recommender systems, stochastic optimization and other
tasks in machine learning \citep[for a review,
see][]{dpps-in-randnla,dpp-ml}. Efficient algorithms for sampling from these
distributions have been proposed both in the CSSP setting \citep[i.e.,
given matrix $\A$; see, e.g.,][]{efficient-volume-sampling,dpp-intermediate}
and in the Nystr\"om setting \citep[i.e., given kernel $\K$;
see, e.g.,][]{rayleigh-mcmc,dpp-sublinear}. The term ``cardinality constrained
DPP'' (also known as a ``k-DPP'' or ``volume sampling'') was introduced by \citet{k-dpp} to
differentiate from standard DPPs which have random cardinality.
Our proofs rely in part on converting DPP bounds to
k-DPP bounds via a refinement of the
concentration of measure argument used by
\citet{bayesian-experimental-design}.

\underline{Beyond worst-case analysis} of algorithms is crucial to understanding the often-noticed gap between practical performance and theoretical guarantees of these algorithms. However, there have been limited number of works in machine learning that undertake finer-grained studies for beyond worst-case analyses. We refer to~\citet{roughgarden2019beyond} for a recent survey of such studies for a few problems in machine learning.~\citet{Mah12} takes an alternative view and studies implicit statistical properties of worst case algorithms. 

\section{Determinantal point processes}
\label{s:dpp}

Since our main results rely on randomized subset selection via
determinantal point processes (DPPs), we provide a brief overview of
the relevant aspects of this class of distributions.
First introduced by \citet{Macchi1975}, a determinantal point process is a
probability distribution over subsets $S \subseteq[n]$, where we use $[n]$
to denote the set $\{1,...,n\}$. The relative probability of a subset being drawn is 
governed by a positive semidefinite (p.s.d.) matrix $\K \in \R^{n\times n}$,
as stated in the definition below, where we use $\K_{S,S}$ to denote
the $|S|\times |S|$ submatrix of $\K$ with rows and columns indexed
by $S$.
\begin{definition}\label{d:dpp}
  For an $n\times n$ p.s.d.~matrix $\K$, define $S\sim
  \DPP(\K)$ as a distribution 
  over all subsets $S\subseteq [n]$ so
  that
  \[\Pr(S)=\frac{\det(\K_{S,S})}{\det(\I+\K)}.\]
A restriction to subsets of size $k$ is denoted as $k$-$\DPP(\K)$.
\end{definition}
DPPs can be used to introduce diversity in the selected set or to
model the preference for selecting dissimilar items, where the
similarity is stated by the kernel matrix $\K$. DPPs are
commonly used in many machine learning applications where these
properties are desired, e.g.,~recommender systems~\citep{Warlop2019}, model
interpretation~\citep{kim:2016MMD}, text
and video summarization~\citep{dpp-video}, and others
\citep{dpp-ml}. They have also played an important role in randomized
numerical linear algebra \citep{dpps-in-randnla}.

Given a p.s.d.~matrix $\K \in \R^{n \times n}$ with eigenvalues
$\lambda_1,...\, \lambda_n$, the size of the set $S \sim \DPP(\K)$
is distributed as a Poisson binomial random variable, namely, the
number of successes in $n$ Bernoulli random trials where the
probability of success in the $i$th trial is given by 
$\frac{\lambda_i}{\lambda_i+1}$. This leads to a simple
expression for the expected subset size:
\begin{align}
\E[|S|] = \sum_i \frac{\lambda_i}{\lambda_i + 1} = \tr(\K(\I + \K)^{-1}).\label{eq:size}
\end{align}
Note that if $S\sim\DPP(\frac1\alpha\K)$, where $\alpha>0$, then
$\Pr(S)$ is proportional to $\alpha^{-|S|}\det(\K_{S,S})$, so rescaling the kernel
by a scalar only affects the distribution of the subset sizes, giving
us a way to set the expected size to a desired value (larger $\alpha$
means smaller expected size).
Nevertheless, it is still often preferrable to restrict the size of
$S$ to a fixed $k$, obtaining a $k$-$\DPP(\K)$ \citep{k-dpp}.

Both DPPs and k-DPPs can be sampled efficiently, with some of the
first algorithms provided by
\citet{dpp-independence},
\citet{efficient-volume-sampling}, \citet{k-dpp} and others. These
approaches rely on an eigendecomposition of the kernel $\K$, at the cost of $O(n^3)$. When
$\K=\A^\top\A$, as in the CSSP, and the dimensions satisfy $m\ll n$, then this can be improved
to $O(nm^2)$. More recently, algorithms that avoid computing the
eigendecomposition have been proposed
\citep{dpp-intermediate,dpp-sublinear,alpha-dpp,rayleigh-mcmc},
resulting in running times of $\Ot(n)$ when given matrix $\K$ and
$\Ot(nm)$ for matrix $\A$, assuming small desired subset size.
See \citet{dppy} for an efficient Python implementation of DPP sampling.

The key property of DPPs that enables our analysis is a
formula for the expected value of the random matrix that is the orthogonal projection onto the
subspace spanned by vectors selected by $\DPP(\A^\top\A)$.
In the special case when $\A$ is a square full rank matrix, the
following result can be derived as a corollary of Theorem 1 by
\citet{randomized-newton}, and a variant for DPPs over
continuous domains can be found as Lemma~8 of
\citet{surrogate-design}. For completeness, we also provide a proof in
Appendix~\ref{a:expectration-proj}.
\begin{lemma}\label{l:expectation_proj}
  For any $\A$ and $S\subseteq [n]$, let $\P_S$ be the  
projection onto the $\mathrm{span}\{\a_i:i\in S\}$. If $S\sim \DPP(\A^\top\A)$, then
\begin{align*}
  \E[\P_S] & = \A(\I+\A^\top\A)^{-1}\A^\top.
\end{align*}
\end{lemma}
Lemma~\ref{l:expectation_proj} implies a simple closed form expression for
the expected error in the CSSP. Here, we use a
rescaling parameter $\alpha>0$ for controlling
the distribution of the subset sizes. Note that it is crucial that we
are using a DPP with random subset size, because the corresponding
expression for the expected error of the fixed size k-DPP is
combinatorial, and therefore much harder to work with.
\begin{lemma}\label{l:expected-error}
  For any $\alpha>0$, if $S\sim
  \DPP(\frac1\alpha\A^\top\A)$, then
  \begin{align*}
    \E\big[\Er_\A(S)\big] =
    \tr\big(\A\A^\top(\I+\tfrac1\alpha\A\A^\top)^{-1}\big) = \E[|S|]\cdot\alpha.
  \end{align*}
\end{lemma}
\begin{proof}
Using Lemma~\ref{l:expectation_proj}, the expected loss is given by:
	\begin{align*}
          \E\big[\Er_\A(S)\big]
          &= \E\big[ \|(\I-\P_S)\A\|_F^2\big]
            = \tr(\A\A^\top\E[\I-\P_S]) \\
          &=\tr\big(\A\A^\top
            (\I-\tfrac1\alpha\A(\I+\tfrac1\alpha\A^\top\A)^{-1}\A^\top)\big)\\
          &\overset{(*)}{=}\tr\big(\A\A^\top(\I+\tfrac1\alpha\A\A^\top)^{-1}\big),
	\end{align*}
where $(*)$ follows from the matrix identity
$(\I+\A\A^\top)^{-1}=\I-\A(\I+\A^\top\A)^{-1}\A^\top$. 
\end{proof}

\section{Proof of Lemma \ref{l:expectation_proj}}
\label{a:expectration-proj}
  We will use the following standard determinantal summation identity
  \citep[see Theorem 2.1 in][]{dpp-ml} which corresponds to computing the
normalization constant $\det(\I+\K)$ for a DPP.
  \begin{lemma}\label{l:normalization}
    For any $n\times n$ matrix $\K$, we have
    \begin{align*}
      \det(\I+\K) = \sum_{S\subseteq[n]}\det(\K_{S,S}).
      \end{align*}
    \end{lemma}
We now proceed with the proof of Lemma \ref{l:expectation_proj}
(restated below for convenience).
  \begin{replemma}{l:expectation_proj}
  For any $\A$ and $S\subseteq [n]$, let $\P_S$ denote the  
projection onto the $\mathrm{span}\{\a_i:i\in S\}$. If $S\sim \DPP(\A^\top\A)$, then
\begin{align*}
  \E[\P_S] & = \A(\I+\A^\top\A)^{-1}\A^\top.
\end{align*}
\end{replemma}
\begin{proof}
Fix $m$ as the column dimension of $\A$ and let $\A_S$ denote the
submatrix of $\A$ consisting of the columns 
indexed by $S$. We have $\P_S=\A_S(\K_{S,S})^\dagger\A_S$, where $^\dagger$
denotes the Moore-Penrose inverse and $\K=\A^\top\A$. Let
$\v\in\R^m$ be an arbitrary vector. When $\K_{S,S}$ is
invertible, then a standard determinantal 
identity states that:
  \begin{align*}
\det(\K_{S,S})\v^\top\P_S\v
    =\det(\K_{S,S})\v^\top\A_S\K_{S,S}^{-1}\A_S^\top\v =
    \det(\K_{S,S}+\A_S^\top\v\v^\top\A_S)-\det(\K_{S,S}). 
  \end{align*}
When $\K_{S,S}$ is not invertible then
$\det(\K_{S,S})=\det(\K_{S,S}+\A_S^\top\v\v^\top\A_S)=0$, because the
rank of $\K_{S,S}+\A_S^\top\v\v^\top\A_S=\A_S^\top(\I+\v\v^\top)\A_S$
cannot be higher than the rank of $\K_{S,S}=\A_S^\top\A_S$. Thus,
\begin{align*}
  \det(\I+\K)\v^\top\E[\P_S]\v
  &= \sum_{S\subseteq[n]:\,\det(\K_{S,S})>0}\det(\K_{S,S})
\v^\top\A_S\K_{S,S}^{-1}\A_S^\top\v\\
  &=\sum_{S\subseteq[n]}\det(\K_{S,S}+\A_S^\top\v\v^\top\A_S)-\det(\K_{S,S})\\
  &=\sum_{S\subseteq[n]}\det\!\big([\K+\A^\top\v\v^\top\A]_{S,S}\big)-\sum_{S\subseteq[n]}\det(\K_{S,S})\\
  &\overset{(*)}{=}\det(\I + \K + \A^\top\v\v^\top\A) - \det(\I+\K)\\
  &=\det(\I+\K)\v^\top\A(\I+\K)^{-1}\A^\top\v,
\end{align*}
where $(*)$ involves two applications of Lemma \ref{l:normalization}. Since the above
calculation holds for arbitrary vector $\v$, the claim follows.
\end{proof}

\section{Proofs omitted from Section \ref{s:upper}}
\begin{replemma}{l:upper}
For any $\A$, $0\leq\epsilon<1$ and $s<k<t_s$, where $t_s=s+\sr_s(\A)$,
suppose that $S\sim\DPP(\frac1\alpha\A^\top\A)$ for
$\alpha=\frac{\gamma_s(k)\opt}{(1-\epsilon)(k-s)}$ and
$\gamma_s(k)=\sqrt{1+\frac{2(k-s)}{t_s-k}\,}$. Then: 
\begin{align*}
  \frac{\E\big[\Er_\A(S)\big]}{\opt}\leq \frac{\Phi_s(k)}{1-\epsilon}
  \quad\text{and}\quad \E[|S|]\leq k-\epsilon\,\frac{k-s}{\gamma_s(k)},
\end{align*}
where $\Phi_s(k) =\big(1+\frac{s}{k-s}\big)\,\gamma_s(k)$.
 \end{replemma}
\begin{proof}
  Let $\lambda_1\geq \lambda_2\geq ...$ be the eigenvalues of
  $\A^\top\A$. Note that scaling the matrix $\A$ by any constant $c$ and
  scaling $\alpha$ by $c^2$ preserves the distribution of $S$ as well
  as the approximation ratio, so without loss of generality, assume
  that $\lambda_{s+1}=1$. Furthermore, using the shorthands
    $l=k-s$ and $r=\sr_s(\A)$, we have $t_s-k=r-l$ and so
    $\gamma_s(k)=\sqrt{\tfrac{r+l}{r-l}}$. We now lower bound the optimum as follows:
  \begin{align*}
\opt = \sum_{j>k}\lambda_j = \sr_s(\A) - \!\sum_{j=s+1}^k\!\lambda_j\
    \geq\  r-l.
  \end{align*}
  We will next define an alternate sequence of eigenvalues which is in
  some sense ``worst-case'', by shifting the spectral mass away from
  the tail. 
  Let $\lambda_{s+1}'=...=\lambda_{k}'=1$, and for $i>k$ set
    $\lambda_i'=\beta\lambda_i$, where $\beta =
    \frac{r-l}{\opt}\leq 1$. Additionally, define:
\begin{align}\alpha'&=\beta\alpha=\frac{\gamma_s(k)(r-l)}{(1-\epsilon)l} =
  \frac{\sqrt{r^2-l^2}}{(1-\epsilon)l},\nonumber\\
  \alpha''&=(1-\epsilon)\frac{\sqrt{r+l}+\sqrt{r-l}}{2\sqrt{r+l}}\alpha'
=\frac{(\sqrt{r+l}+\sqrt{r-l})\sqrt{r-l}}{r+l-(r-l)}=\frac{\sqrt{r-l}}{\sqrt{r+l}-\sqrt{r-l}}.
            \label{eq:alpha}
  \end{align}
and note that $\alpha''\leq\alpha'\leq \alpha$. Moreover, for $s+1\leq i\leq k$,
we let $\alpha_i'=\alpha''$, while for $i>k$ we set $\alpha_i'=\alpha'$.
We proceed to bound the expected subset size $\E[|S|]$ by converting
all the eigenvalues from $\lambda_i$ to $\lambda_i'$ and $\alpha$ to
$\alpha_i'$, which will allow us to easily bound the entire expression:
  \begin{align}
    \E[|S|] = \sum_{i}\frac{\lambda_i}{\lambda_i+\alpha} 
\leq s + \sum_{i=s+1}^{k}\frac{\lambda_i}{\lambda_i+\alpha_i'}
    + \sum_{i>k}\frac{\beta\lambda_i}{\beta\lambda_i+\beta\alpha}
    \leq s + \sum_{i=s+1}^k  \frac{\lambda_i'}{\lambda_i'+\alpha''} +
    \sum_{i>k}\frac{\lambda_i'}{\lambda_i'+\alpha'}.\label{eq:converting}
  \end{align}
  We bound each of the two sums separately starting with the first one:
  \begin{align}
    \sum_{i=s+1}^k  \frac{\lambda_i'}{\lambda_i'+\alpha''}
    =\frac l{1+\alpha''} = l- \frac {l}{1+\frac1{\alpha''}} =
    l-\frac{l}{1+\frac{\sqrt{r+l}-\sqrt{r-l}}{\sqrt{r-l}}} = l - \frac{l\sqrt{r-l}}{\sqrt{r+l}}.\label{eq:sum1}
  \end{align}
To bound the second sum, we use  the fact that
$\sum_{i>k}\lambda_i'=\beta\,\opt=r-l$, and obtain:
  \begin{align}
    \sum_{i>k}\frac{\lambda_i'}{\lambda_i'+\alpha'}
    \le \frac1{\alpha'} \sum_{i>k} \lambda_i' 
 = \frac{r-l}{\alpha'}
=                  (1-\epsilon)\frac{l\sqrt{r-l}}{\sqrt{r+l}}.\label{eq:sum2}
  \end{align}
Combining the two sums, we conclude that $\E[|S|]\leq s+l - \epsilon\,l
\sqrt{\tfrac{r-l}{r+l}}=k-\frac{\epsilon\,l}{\gamma_s(k)}$. Finally, Lemma \ref{l:expected-error} yields:
    \begin{align*}
      \frac{\E\big[\Er_\A(S)\big]}{\opt} = \frac{\E[|S|]\cdot\alpha}{\opt}\leq
\frac{k}{k-s}\,\frac{\gamma_s(k)}{1-\epsilon} = \frac{\Phi_s(k)}{1-\epsilon},
    \end{align*}
    which concludes the proof.
  \end{proof}

\begin{replemma}{l:concentration}
Let $S$ be sampled as in Lemma \ref{l:upper} with $\epsilon\leq \frac12$.
If $s+\frac7{\epsilon^4}\ln^2\!\frac1\epsilon \leq k\leq t_s-1$, then
$\Pr(|S|>k) \leq \epsilon$.
\end{replemma}
\begin{proof}
 Let $p_i=\frac{\lambda_i'}{\lambda_i'+\alpha_i'}$ be the 
      Bernoulli probabilities for
  $b_i\sim\mathrm{Bernoulli}(p_i)$
  and $X =\sum_{i>s} b_i$, where $\lambda_i'$ and $\alpha_i'$ are as
  defined in the proof of Lemma \ref{l:upper}. Note that $|S|$ is
  distributed as a Poisson binomial random variable such that the success
  probability associated with the $i$th eigenvalue is upper-bounded by
  $p_i$ for each $i>s$. It follows that $\Pr(|S|\!>\!k)\leq \Pr(X\!>\!l)$, where $l=k-s$.
Moreover, letting $r=\sr_s(\A)$, in the proof of Lemma \ref{l:upper} we showed that:
  \begin{align*}
k - \E[X] \geq \epsilon\,\frac{ l\sqrt{r-l}}{\sqrt{r+l}},
  \end{align*}
  and furthermore, using the derivations in \eqref{eq:sum1} and
  \eqref{eq:sum2} together with the formula $\Var[b_i]=p_i(1-p_i)$, we obtain that:
  \begin{align*}
    \Var[X]&\leq \sum_{i=s+1}^{k}(1-p_i) + \sum_{i>k}p_i
    \leq \frac{l\sqrt{r-l}}{\sqrt{r+l}}+
      (1-\epsilon)\frac{l\sqrt{r-l}}{\sqrt{r+l}} = (2-\epsilon) \frac{l\sqrt{r-l}}{\sqrt{r+l}}.
  \end{align*}
    Using Theorem 2.6 from \citet{ChungLu2006book} with $\lambda=\epsilon\,\frac{
    l\sqrt{r-l}}{\sqrt{r+l}}$, we have:
  \begin{align*}
    \Pr(|S|>k)&\leq \Pr(X>l)\leq \Pr(X > \E[X]+\lambda)
    \leq \exp\Big(-\frac{\lambda^2}{2(\Var[X]+\lambda/3)}\Big)\\
    &\leq \exp\Big(-\frac{\lambda^2}{2(\frac{2-\epsilon}\epsilon
      \lambda+\lambda/3)}\Big)\leq\exp(-\epsilon\lambda /4)
      =\exp\Big(-\frac{\epsilon^2l\sqrt{r-l}}{4\sqrt{r+l}}\Big). 
  \end{align*}
  Note that since $7\leq l\leq r-1$, we have $l\frac{\sqrt{r-l}}{\sqrt{r+l}}\geq
\frac{l}{\sqrt{2l+1}}\geq \frac{7}{16}\sqrt l$, so by simple algebra
it follows that for $l\geq \frac 7{\epsilon^4}\ln^2\frac1\epsilon$,
we have $l\frac{\sqrt{r-l}}{\sqrt{r+l}}\geq
\frac4{\epsilon^2}\ln\frac1\epsilon$ and therefore
$\Pr(|S|>k)\leq \epsilon$. 
\end{proof}
\begin{replemma}{l:monotonic}
  For any $\A\in\R^{m\times n}$, $k\in[n]$ and $\alpha>0$, if
  $S\sim\DPP(\frac1\alpha\A^\top\A)$ and $S'\sim k$-$\DPP(\A^\top\A)$, then
  \begin{align*}
    \E\big[\Er_\A(S')\big] \leq \E\big[\Er_\A(S)\mid |S|\leq k\big].
  \end{align*}
\end{replemma}
\begin{proof}
Let $\lambda_1\geq\lambda_2\geq ...$ denote the eigenvalues of
$\A^\top\A$ and let $e_k$ be the $k$th elementary symmetric polynomial
of $\A$:
\begin{align*}
e_k = \sum_{T:|T|=k}\det(\A_T^\top\A_T) = \sum_{T:|T|=k}\prod_{i\in T}\lambda_i.
\end{align*}
Also let $\bar{e}_k=e_k/{n\choose k}$ denote the $k$th
elementary symmetric mean. Newton's inequalities imply that:
\begin{align*}
1\geq \frac{\bar e_{k-1}\bar e_{k+1}}{\bar e_k^2} =
  \frac{e_{k-1}e_{k+1}}{e_k^2}\,\frac{{n\choose k}}{{n\choose
  k-1}}\frac{{n\choose k}} {{n\choose k+1}} =
  \frac{e_{k-1}e_{k+1}}{e_k^2}\,\frac{n+1-k}{k}\,\frac{k+1}{n-k}. 
\end{align*}
The results of
\citet{pca-volume-sampling} and \citet{more-efficient-volume-sampling} establish
that $\E[\Er_\A(S)\mid |S|=k] = (k+1)\frac{e_{k+1}}{e_k}$, so it
follows that:
\begin{align}
  \frac{\E[\Er_\A(S)\mid |S|=k]}{\E[\Er_\A(S)\mid|S|=k-1]}
  = \frac{k+1}{k}\,\frac{e_{k+1}e_{k-1}}{e_k^2}\leq
  \frac{n-k}{n+1-k}\leq 1.\label{eq:decreasing-error}
\end{align}
Finally, note that $\E[\Er_\A(S)\mid|S|\leq k]$ is a weighted average
of components $\E[\Er_\A(S)\mid|S|= s]$ for $s\in[k]$, and
\eqref{eq:decreasing-error} implies that the smallest of those components
is associated with $s=k$. Since the weighted average is lower bounded by the
smallest component, this completes the proof.
  \end{proof}
  
\section{Proof of Theorem~\ref{t:decay}}\label{a:decay}
Before showing Theorem \ref{t:decay}, we give an additional lemma
which covers the corner case of the theorem when $k$ is close to $n$.
  \begin{lemma}\label{l:upper_conditionnum}
Given $\A\in\R^{m\times n}$ and $s< k< n$, let
$\lambda_1\geq...\geq \lambda_n>0$ be the eigenvalues of $\A^\top\A$.
If $S\sim k$-$\DPP(\A^\top\A)$ and we let $b=\min\{k-s,n-k\}$, then
for any $0<\epsilon\leq \frac12$ we have
    \begin{align*}
\frac{\E[\Er_\A(S)]}{\opt} \leq \big(1-\ee^{-\frac{\epsilon^2b}{10}}\big)^{-1}(1-\epsilon)^{-1}\Psi_{s}(k),
    \end{align*}
    where $\Psi_s(k) = \frac{\lambda_{s+1}}{\lambda_n}\,\big(1+\frac s{k-s}\big)$.
    \end{lemma}
    \begin{proof}
      Let $\alpha=\frac{\lambda_{s+1}}{(1-\epsilon)\lambda_n}\frac{\opt}{k-s}$.
Note that $\opt=\sum_{i>k}\lambda_i\geq (n-k)\lambda_n$. Define
$b_i\sim\mathrm{Bernoulli}(\frac{\lambda_i}{\lambda_i+\alpha})$ and
let $X=\sum_{i>s}b_i$. We have: 
\begin{align*}
          \E[X]&= \sum_{i>s}\frac{\lambda_i}{\lambda_i+\alpha}\\
          &\leq 
            \frac{(n-s)\lambda_{s+1}}{\lambda_{s+1}+
            \frac{\lambda_{s+1}}{\lambda_n}\frac{(n-k)\lambda_n}{(1-\epsilon)(k-s)}}\\
          &= \frac{1}{\frac1{n-s} +
            \frac1{(1-\epsilon)(k-s)}\,\frac{n-k}{n-s}} \\
          &= \frac{1}{\frac1{n-s} +
            \frac{1}{(1-\epsilon)(k-s)}(1-\frac{k-s}{n-s})}\\
          &=\frac{1}{\frac{1}{(1-\epsilon)(k-s)} -
            \frac{\epsilon}{1-\epsilon}\frac1{n-s}}\\
          &=\frac{1-\epsilon}{\frac1{k-s}-\frac\epsilon{n-s}}.
\end{align*}
Let $S'\sim\DPP(\frac1\alpha\A^\top\A)$. It follows that
\begin{align*}
  k-\E[|S'|]&\geq k -(s+\E[X])\\
  & \geq (k-s) -
  \frac{1-\epsilon}{\frac1{k-s}-\frac\epsilon{n-s}}\\
&=(k-s)\bigg(1-\frac{1-\epsilon}{1-\epsilon\,\frac{k-s}{n-s}}\bigg)\\
&=(k-s)\,\frac{\epsilon - \epsilon\,
\frac{k-s}{n-s}}{1-\epsilon\,\frac{k-s}{n-s}}\\
             &\geq \epsilon\,(k-s)\bigg(1-\frac{k-s}{n-s}\bigg)\\
             &=\epsilon\cdot\frac{(k-s)(n-k)}{n-s}\\
  &\geq \frac\epsilon2\cdot\min\{k-s,n-k\}.
\end{align*}
From this, it follows that:
\begin{align*}
  \frac{\E[\Er_\A(S')]}{\opt}=\frac{\E[|S|]\cdot\alpha}{\opt}\leq
(1-\epsilon)^{-1}  \frac k{k-s}\frac{\lambda_{s+1}}{\lambda_n} =
(1-\epsilon)^{-1}\Big(1+\frac s{k-s}\Big)\frac{\lambda_{s+1}}{\lambda_n}.
\end{align*}
We now give an upper bound on $\Pr(|S'|>k)$ by considering two
cases.

\textbf{Case 1}: $k-s\leq n-k$. Then, using $\lambda=\epsilon(k-s)/2$,
we have $(k-s)-\E[X]\geq \lambda$, so using Theorem 2.4
from \citet{ChungLu2006book}, we get:
\begin{align*}
  \Pr(|S'|>k)\leq \Pr(X>k-s)\leq \Pr(X>\E[X]+\lambda)\leq
  \ee^{-\frac{\lambda^2}{2(k-s)}}
  =\ee^{-\epsilon^2(k-s)/8}.
\end{align*}
\textbf{Case 2}: $k-s>n-k$. Then, using Theorem 2.4
from \citet{ChungLu2006book} with $\lambda=k-\E[|S'|]=\frac{\epsilon(n-k)}2
+ \Delta$, where $\Delta>0$, we get:
\begin{align*}
  \Pr(|S'|>k)
  &= \Pr(n-|S'|<n-k)\\
 &\leq \exp\Big(-\frac{\lambda^2}{2\,\E[n-|S'|]}\Big) \\
  &=\exp\Big(-\frac\lambda2\,\frac{\frac\epsilon2(n-k)
    + \Delta}{n-k+\frac\epsilon2(n-k)+\Delta}\Big)\\
  &\leq\exp\Big(-\frac\lambda2\,\frac{\frac\epsilon2(n-k)}
    {n-k+\frac\epsilon2(n-k)}\Big)\\
  &=\exp\Big(-\frac{\epsilon^2(n-k)}{8(1+\epsilon/2)}\Big)\\
  &\overset{(*)}{\leq} \exp\Big(-\frac{\epsilon^2(n-k)}{10}\Big),
\end{align*}
where in $(*)$ we used the fact that $\epsilon\in(0,\frac12)$. Now,
the result follows easily by invoking Lemma \ref{l:monotonic}:
\begin{align*}
    \E\big[\Er_\A(S)\big] &\leq \E\big[\Er_\A(S')\mid|S'|\leq k\big]
  \leq \frac{\E\big[\Er_\A(S')\big]}{\Pr(|S'|\leq k)}
  \\
  &\leq \big(1-\ee^{-\frac{\epsilon^2b}{10}}\big)^{-1}(1-\epsilon)^{-1}
\frac{\lambda_{s+1}}{\lambda_n}\Big(1+\frac s{k-s}\Big)\cdot\opt,
\end{align*}
which completes the proof.
\end{proof}
Note that since $b\geq 1$, setting $\epsilon=\frac12$ in Lemma
\ref{l:upper_conditionnum} yields the following simpler (but usually
much weaker) bound:
\begin{align*}
  \frac{\E[\Er_\A(S)]}{\opt}
  \leq
  2\big(1-\ee^{-\frac1{40}}\big)^{-1}\Psi_s(k)
  \leq 82\, \Psi_s(k).
\end{align*}

\begin{reptheorem}{t:decay}
Let $\lambda_1\!\geq\!\lambda_2\!\geq\!...$ be the
eigenvalues of $\A^\top\A$. There is an absolute constant $c$ such that
for any $0\!<\!c_1\!\leq\!c_2$, with $\gamma=c_2/c_1$, if:\\[2mm]
\textnormal{1.} (\textbf{polynomial spectral decay}) $c_1i^{-p}\!\leq\!
  \lambda_i\!\leq\! c_2i^{-p}$ $\forall_i$, with $p>1$, then $S\sim
k$-$\DPP(\A^\top\A)$ satisfies
  \begin{align*}
    \frac{\E[\Er_\A(S)]}{\opt}\leq c \gamma p. 
  \end{align*}
\textnormal{2.} (\textbf{exponential spectral decay}) $c_1(1\!-\!\delta)^{i}\leq \lambda_i\leq
  c_2(1\!-\!\delta)^{i}$ $\forall_i$, with $\delta\in(0,1)$, then $S\sim
k$-$\DPP(\A^\top\A)$ satisfies
  \begin{align*}
    \frac{\E[\Er_\A(S)]}{\opt}\leq c\gamma(1+ \delta k).
  \end{align*}
 \end{reptheorem}
\begin{proof}

\textbf{  (1) Polynomial decay.} We provide the proof by splitting it into two cases.

  \textbf{Case 1(a)}: $\left(\frac{k+1}{n}\right)^{p-1} \leq \frac{1}{2} $
  
  We can use upper and lower integrals to bound the sum $\sum_{i\geq s}\frac{1}{i^p}$ as:
  \[ \int_{x\geq (s+1)}\frac{1}{i^p}dx \leq \sum_{i\geq
      s}\frac{1}{i^p} \leq \int_{x\geq s}\frac{1}{i^p}dx
    \implies \sum_{i=s+1}^n\frac{1}{i^p}
    \geq \frac{(s+2)^{1-p}}{p-1} -  \frac{(n+1)^{1-p}}{p-1}.\] 
  We lower bound the stable rank for $s\leq k$ using the upper/lower bounds on
  the eigenvalues and the condition for Case 1(a):
  \begin{align*}
    \sr_s(\A)&=\frac{\sum_{i=s+1}^n\lambda_i}{\lambda_{s+1}}\\
    &\geq
    \frac{c_1}{c_2}\bigg(\frac{(s+2)^{1-p}}{(p-1)(s+1)^{-p}}-
    \frac{(n+1)^{1-p}}{(p-1)(s+1)^{-p}}\bigg) \\
    &=\frac1\gamma\bigg(\frac{s+2}{p-1} \Big(1 - \frac1{s+2}\Big)^p
      -  \frac{s+1}{p-1}\Big(\frac{s+1}{n+1}\Big)^{p-1}\bigg)\\
    &\geq \frac1\gamma\bigg(\frac{s+2}{p-1} - 1
      -  \frac{s+1}{p-1}\cdot\frac12\bigg) = \frac1{2\gamma}\,\frac{s+1}{p-1}-\frac1\gamma.
  \end{align*}

  Further using $u =
  k-s$, we can call upon Theorem~\ref{t:upper} to get,  
  \begin{align*}
    \Phi_s(k) &\leq \frac{k}{u}\sqrt{1+\frac{2u}{ \sr_s-u }\,} \leq
    \frac{k}{u} + \frac{k}{  \frac{1}{2\gamma} \frac{s+1}{p-1} -
      \gamma^{-1} - u } = \frac{k}{u} + \frac{(2p-2)k}{
                \gamma^{-1}(s+1 - 2 p +2) - (2p-2)u } \\
    &\leq \frac{k}{u} +
    \frac{(2p-2 + \gamma^{-1})k}{   \gamma^{-1}(k+3 -2 p)  - (2p-2 +
    \gamma^{-1})u }
    \end{align*}
  
  Optimizing over $u$, we see that the minimum is reached for $u =
  \hat{u} = \frac{k+3-2p}{2\gamma(2p-2+\gamma^{-1})}$ which achieves the value
  $ \frac{4(\gamma(2p-2)+1)k}{k+3 -2 p}$ which is upper bounded by
  $\frac{12\gamma pk}{(k-2p)}$.   
  
  We assume $k \geq \hat{u} > 60p > 60$. If not,~\citet{pca-volume-sampling} ensure an upper bound of $(k+1) \leq 60p+1 < 61p$. With $p < k/60$, we get:
  
  \[ \frac{12\gamma pk}{k-2p} \leq \frac{12\gamma pk}{k-k/30} =
    \frac{12\gamma p}{1-1/30} \leq \frac{360}{29}\gamma  p.  \] 
  
  Since we assumed that $\hat u> 60$, then $k-s> \frac{7}{\epsilon^4}
  \ln^2\frac{1}{\epsilon}$ for $\epsilon = 0.5$ which means
  $(1+2\epsilon)^2 \leq 4$, which makes the approximation ratio upper
  bounded by $ \frac{1440}{29}\gamma p  $. The overall bound
  thus becomes $61 \gamma p$.

  \bigskip 
  \textbf{Case 1(b)}: $\left(\frac{k+1}{n}\right)^{p-1} >\frac12$
  
  From Lemma~\ref{l:upper_conditionnum}, we know that the
  approximation ratio is upper bounded by constant factor times
  $\Psi_s(k)=\frac{\lambda_{s+1}}{\lambda_n}\,\frac{k}{k-s}$. Consider,  
  
  \begin{align*}
    \Psi_s(k) &= \frac{\lambda_{s+1}}{\lambda_n}\frac{k}{k-s}\leq
      \gamma\frac{n^p}{(s+1)^p}\frac{k}{k-s} =\gamma
      \left(\frac{n}{k+1}\right)^{p-1}
      \frac{k+1}{n}\frac{{(k+1)^p}}{{(s+1)^p}} \frac{k}{k-s}
    \leq 2\gamma\left(\frac{k+1}{s+1}\right)^p \frac{k}{k-s},
  \end{align*}
  
  which holds true for all $s \leq k$, and is optimized for $s=
  \hat{s}=\frac{pk-1}{p+1}$. We get that the approximation
  ratio is bounded as: 
  
  \[\Psi_s(k) \leq
    \gamma\frac{k(p+1)}{k+1}\left(\frac{p+1}{p}\right)^p \leq
    e\gamma(p+1)\leq 2e\gamma p. \] 
  
  Combining in the factor based on $\epsilon$ in
  Lemma~\ref{l:upper_conditionnum}, we get an upper bound of
  $164e\gamma p$ that is larger than the bound obtained in the case
  1(a) above and hence covers all the subcases. 
  
\textbf{  (2) Exponential decay.}

We  first lower bound the stable rank of $\A$ of order $s$:
  \[ \sr_s(\A) = \sum_{j>s} \lambda_j/\lambda_{s+1} \geq
    \frac{c_1(1-(1-\delta)^{n-s})/\delta }{c_2} =
    \frac{1-(1-\delta)^{n-s}}{\gamma\delta}. 
  \]
  
  We present the proof by considering two subcases separately : when
  $k\leq n-  \frac{\ln 2}{\delta}$ and $k>n-  \frac{\ln 2}{\delta}$. 
  
  \textbf{Case 2(a):}  $k\leq n-  \frac{\ln 2}{\delta}$. From the
  assumption, letting $s\leq k$  we have
  \begin{align*}
& s \leq n - \frac{\ln 2}{\delta}  \\
    \implies  & s \leq n - \frac{\ln 2}{\ln \frac{1}{1-\delta}} \\
    \implies  & (n-s)\ln \frac{1}{1-\delta} \geq \ln 2\\
    \iff & 1 - (1-\delta)^{n-s} \geq \frac{1}{2}   \\
    \implies & \sr_s(\K) \geq \frac{1}{2\gamma \delta},
  \end{align*}
  
  where the second inequality follows because $\frac{x}{1+x} \leq \ln (1+x)$ with $x = \delta/(1-\delta)$.  
  
  We will use $u=k-s$. From Theorem~\ref{t:upper}, using $\sr_s \geq
  \frac{1}{2\gamma\delta}$ we have the following upper bound:   
  \[ \Phi_s(k) \leq \frac{k}{u} \left(1+ \frac{2\gamma\delta u}{1 -
        2\gamma\delta u}\right)  = \frac{k}{u}\cdot
    \frac{1}{1-2\gamma\delta u}.\]
  RHS is minimized for $\hat{u}=\frac{1}{4\gamma\delta}$. We let
  $\epsilon=0.5$ and assume that
  $\hat{u}\geq 60$ which is bigger than $\frac{7}{\epsilon^4} \ln^2 \frac{1}{\epsilon}$.
If not, then $\delta\geq\frac4{60\gamma}>\frac1\gamma$ and
the  worst-case bound of \citet{pca-volume-sampling} ensures that the
approximation factor is no more than $k+1 \leq \gamma(1+\frac1\gamma
k)\leq \gamma(1+\delta k)$. By a similar argument we can assume that
$k\geq 60$.
  
If  $k  \leq \hat{u}$, in this case we can set $s=0$, i.e., $u = k$, obtaining $\Phi_s(k) \leq \frac{1}{1- 2\gamma\delta k}\leq 2$. And so the
approximation ratio is bounded by $(1+2\epsilon)^2 \cdot 2 \leq 8$. On the other
hand, if $k > \hat{u}$, we can set $u=\hat{u}$, which implies
$\Phi_s(k) \leq 8\gamma\delta k$, and so the approximation ratio is
bounded by $ 32
\gamma\delta k$. The overall bound is thus $61\gamma(1+\delta k)$
covering all possible subcases. 
  
  \bigskip
  \textbf{Case 2(b)}: $k > n - \frac{\ln 2}{\delta}$. We make use of
  Lemma~\ref{l:upper_conditionnum} for the case when $k$ is close to
  $n$. The approximation guarantee uses:  
   
  \[\Psi_s(k) = \frac{\lambda_{s+1}}{\lambda_n} \frac{k}{k-s},\]
  
where $s <k$. For our bound, we choose $s = \lfloor k- \frac{\ln
  2}{\delta}\rfloor$. This implies that $n-s < \frac{2\ln
  2}{\delta}+1=\frac{\delta+\ln4}{\delta}$. It
follows that
  \begin{align*}
    & \frac{\lambda_{s+1}}{\lambda_n}  \leq
     \frac{ \gamma}{(1-\delta)^{n-s}} \leq
      \frac{\gamma}{(1-\delta)^{(\delta+\ln 4)/\delta}} =
      \gamma\left[(1-\delta)^{-\frac{1}{\delta}}\right] ^ {\delta+\ln 4} \leq
      \gamma e^{\frac{\delta+\ln 4}{1-\delta}}. 
  \end{align*}
  If $\delta\geq \frac1{20}$, then the worst-case result of
  \citet{pca-volume-sampling} suffices to show that the approximation ratio is
  bounded by $k+1\leq 20(1+\delta k)$, so assume that
  $\delta<\frac1{20}$. Then we have $\ee^{\frac{\delta+\ln 4}{1-\delta}}<5.$
  Combining this with the fact that $\frac{k}{k-s} \leq \frac{\delta k}{\ln 2}$,
we obtain:
  \[\Phi_s(k) \leq \frac{5\gamma\delta k}{\ln 2}. \]
  
  Combining with factor based on $\epsilon$ in
  Lemma~\ref{l:upper_conditionnum}, we get $82\cdot
  \frac{5\gamma\delta k}{\ln 2}$. Thus, the bound of $\frac{82\cdot
    5}{\ln 2}\gamma(1+\delta k)$ holds in all cases, completing the proof.
\end{proof}

\section{Proof of Lemma \ref{l:lower}}
\begin{replemma}{l:lower}
Fix $\delta\in(0,1)$ and consider unit vectors
$\a_{i,j}\in\R^m$ in general position, where $i\in[t]$, $j\in[l_i]$, such that
$\sum_j\a_{i,j}=0$ for each $i$, and for any $i,j,i',j'$, if $i\ne i'$
then $\a_{i,j}$ is orthogonal to $\a_{i'\!,j'}$. Also, let unit vectors
$\{\v_i\}_{i\in[t]}$ be orthogonal to  each other and to all
$\a_{i,j}$. There are positive scalars $\alpha_{i},\beta_i$ for $ i \in [t]$
such that matrix $\A$ with columns 
$\alpha_{i}\a_{i,j}+\beta_i\v_i$ over all $i$ and $j$ satisfies: 
\begin{align*}
  \min_{|S|=k_i}\frac{\Er_\A(S)}{\mathrm{OPT}_{k_i}}\geq
  (1-\delta)l_i,\quad\text{for }k_i=l_1+...+l_i-1.
\end{align*}
 \end{replemma}

\begin{proof}
Say $\Abh_i$ is the matrix obtained
  by stacking all the $\a_{i,j}$ and let $\lambda_{i,1}\geq
  \lambda_{i,2}\geq ...\geq \lambda_{i,l_i-1}$ denote the non-zero eigenvalues of
  $\Abh_i^\top\Abh_i$. We write $\tilde\a_{i,j} = 
  \alpha_i\a_{i,j}+\beta_i \v_i$ and note that for each $i$,
  $\one_{l_i}$ is an eigenvector of $\Abh_i^\top\Abh_i$ with
  eigenvalue $0$. Further, $\A^\top\A$ is a block-diagonal matrix
with blocks $\B_i=\alpha_i^2\Abh_i^\top\Abh_i  + \beta_i^2 \one_{l_i} \one_{l_i}^\top$: 
  \begin{align*}
    \A^\top \A = \begin{bmatrix}
\B_1&\zero
      &\zero\\
      \zero &\ddots &\zero\\
      \zero &\zero &\B_t
      \end{bmatrix}
  \end{align*}
  Therefore, the eigenvalues of $\A^\top\A$ are
  $\alpha_1^2\lambda_{1,1},...,\alpha_1^2\lambda_{1,l_1-1},\beta_1^2l_1,
  ...,
  \alpha_t^2\lambda_{t,1},...,\alpha_t^2\lambda_{t,l_t-1},\beta_t^2l_t$,
  and so we can always choose the parameters so that
  $\alpha_i\gg\beta_i\gg\alpha_{i+1}$ for each $i$, ensuring that
  these eigenvalues are in decreasing order.
Let us fix an arbitrary $c\in[t]$. From the above, it follows that for
$k_c=\big(\sum_{i\leq c}l_i\big)-1$ we have:
  \begin{align*}
  \mathrm{OPT}_{k_c} =l_c\beta_c^2 + \sum_{i>c}\tr(\B_i)=
l_c\beta_c^2 + \phi_c,  
\end{align*}
where we use $\phi_c =\sum_{i>c}\tr(\B_i) $ as a shorthand. 
Since the centroid of $\{\tilde\a_{c,1}, \ldots, \tilde\a_{c,l_c} \}$ is $\beta \v_c$, we can write 
$\tilde\a_{c, l_c} = l_c\beta\v_c - \sum_{j< l_c} \tilde\a_{c,j}$.
For selecting the set $S\subset [n]$ of size $k_c$, since
$\alpha_i\gg\alpha_{i+1}$, we can assume without loss of
generality that $S$ does not select any vectors $\tilde\a_{i,j}$ such
that $i>c$ and does not drop any such that $i<c$, and so for some
$j'\in[l_c]$ we let $S_{j'}$ be the index set such that $\P_{S_{j'}}$ is the
projection onto the span of $\Big(\bigcup_{i<c} \bigcup_j \{\tilde\a_{i,j}\}\Big) \cup
\{\tilde\a_{c,1},\ldots,\tilde\a_{c,l_c}\}\backslash\{\tilde\a_{c,j'}\}$.  We now lower bound the
squared projection error of that set:
\begin{align*}
  \Er_{\A}(S_{j'})
  &= \| \tilde\a_{c,j'}- \P_{S_{j'}} \tilde\a_{c,j'}  \|^2
    +  \sum_{i>c}\sum_{j=1}^{l_i}\| \tilde\a_{i,j} - \P_{S_{j'}}
    \tilde\a_{i,j} \|^2\\ 
  &=\bigg\|l_c\beta\v_c - \sum_{j<l_c}\tilde\a_{c,j}- \P_{S_{j'}}
    \Big(l_c \beta\v_c - \sum_{j< l_c}\tilde\a_{c,j} \Big)\bigg\|^2
    +
    \sum_{i>c}\sum_{j=1}^{l_i}\| \tilde\a_{i,j} \|^2
  \\ 
  &= l_c^2\beta^2 \|\v_c - \P_{S_{j'}}\v_c\|^2 + \phi_c \\
  & = l_c (\mathrm{OPT}_{k_c}- \phi_c) \|\v_c - \P_{S_{j'}}\v_c\|^2 + \phi_c \\
  & \geq l_c\mathrm{OPT}_{k_c} \|\v_c - \P_{S_{j'}}\v_c\|^2 - l_c \phi_c. 
\end{align*}	
Note that $\lim_{\beta\rightarrow 0}\P_{S_{j'}}\v_c = \zero$ because $\v_c$
is orthogonal to the subspace spanned by $S_{j'}$, so we can choose
$\beta_c$ small enough so that
$\|\v-\P_{S_{j'}}\v\|^2\geq 1-\frac\delta2$ for each $j'\in[l_c]$.
Furthermore, we have
\begin{align*}
  \phi_c = \sum_{i>c}\tr(\B_i) =
  \sum_{i>c}\alpha_i^2l_i+\beta_i^2l_i\leq 2\alpha_{c+1}^2\sum_{i>c}l_i,
\end{align*}
So, if we ensure that $\alpha_{c+1}^2\leq \frac{\delta}{4}
l_c\beta_c^2/(\sum_{i>c}l_i)$, then:
\begin{align*}
  l_c\phi_c \leq 2l_c\alpha_{c+1}^2\sum_{i>c}l_i\leq \frac\delta
  2\cdot l_c^2\beta^2\leq \frac\delta 2 l_c\cdot\textsc{OPT}_{k_c},
\end{align*}
which implies that $\Er_\A(S_{j'})\geq
(1-\delta)l_c\textsc{OPT}_{k_c}$. Note that all the conditions
we required on $\alpha_i$ and $\beta_i$ can be satisfied by a
sufficiently quickly decreasing sequence
$\alpha_1\gg\beta_1\gg\alpha_2\gg\beta_2\gg...\gg\alpha_t\gg\beta_t>0$,
which completes the proof.
\end{proof}

\section{Proof of Corollary~\ref{c:multiple-descent}}
\label{s:proofmultidescent}
\begin{repcorollary}{c:multiple-descent}
  For $t\in\N$ and $\delta\in(0,1)$, there is a
  sequence $k_1^l<k_1^u<k_2^l<k_2^u<...<k_t^l<k_t^u$ and
  $\A\in\R^{m\times n}$ such that for any $i\in[t]$:
  \begin{align*}
    \min_{S:|S|=k_i^l}\frac{\Er_\A(S)}{\textsc{OPT}_{k_i^l}}&\leq 1+\delta
                                 \quad\text{and }\\
    \min_{S:|S|=k_i^u}\frac{\Er_\A(S)}{\textsc{OPT}_{k_i^u}} &\geq
    (1-\delta)(k_i^u+1).
  \end{align*}
\end{repcorollary}
\begin{proof}
We will use Theorem~\ref{t:lower} to construct the matrix $\A$ using
the sequence we build below to make sure the upper and lower bounds
are satisfied. Theorem~\ref{t:lower} uses Lemma~\ref{l:lower} to
construct the matrix $\A$ which has a ``step" eigenvalue profile
i.e. there are multiple groups of eigenvalues and in each group the
eigenvalue is constant (each group corresponds to a regular simplex, see
Section \ref{s:lower}). Below we consider a single such group that
starts at $s = k^u_i$ and ends at $w=k^u_{i+1}$, and we let $k = k^l_{i+1}$,
for any $i \in \{0,\ldots,t-1\}$, with $k^u_0=0$.   

Theorem~\ref{t:upper} implies that there is a set $S$ with an upper bound on the
approximation factor $\Er_\A(S)/\text{OPT}_k$ of
$(1+2\epsilon)^2\big(1+ \frac{s}{k-s}\big)\big( 1+ \frac{k-s}{t_s-
    k } \big)$. Consider the following three conditions to ensure that each
of the three terms in the above approximation factor is less than $(1+\delta_1)$ where $\delta_1 =
\delta/7$: 

\begin{enumerate}
\item $\epsilon \leq \frac{ (1+\delta_1 )^{1/2} - 1}{2} \implies
  (1+2\epsilon)^2 \leq (1+\delta_1) $. Let $\tau_\epsilon =
  \frac{7}{\epsilon^4} \ln^2 \frac{1}{\epsilon}$, where $\epsilon$ is
  chosen so as to satisfy the above condition. 
\item  $k \geq \frac{s}{\delta_1} + s + \tau_\epsilon$ ensures that $(1+\frac{s}{k-s}) \leq (1+\delta_1)$ and that $k-s \geq \tau_\epsilon$.
\item $w \geq k(1+\frac{1}{\delta_1}) +1$.
  \end{enumerate}
  To see the usefulness of condition 3, note that each group of
  vectors in column set of  $\A$ constructed from Theorem \ref{t:lower} form a
  shifted regular simplex. A regular simplex has the smallest
  eigenvalue $0$ and the rest of the eigenvalues are all $(w-s)
  \alpha^2/(w-s-1)$, where $\alpha$ is the length of each of the
  $(w-s)$ vectors in the simplex. Thus, we can lower bound the stable
  rank of the shifted simplex as $\sr_s(\A) \geq
  \frac{(w-s)\alpha^2}{(w-s) \alpha^2} (w-s-1) = (w-s-1)$. From
  condition 3: 
  \[  w \geq k(1+\frac{1}{\delta_1}) +1  \implies s + \sr_s(\A) \geq k(1+\frac{1}{\delta_1}) \implies t_s \geq k(1+\frac{1}{\delta_1}) -\frac{s}{\delta_1} \implies 1+ \frac{k-s}{t_s-k}  \leq (1+\delta_1).\]	

Thus if all the above three conditions are satisfied, the approximation ratio can be upper bounded by $(1+\delta_1)^3 \leq (1+\delta)$, since $\delta_1 = \delta/7$.

Similarly for the lower bound, we will need condition 4 below.
\begin{enumerate}
  \setcounter{enumi}{3}

	\item $w \geq \frac{ 2s}{\delta} +\frac{2}{\delta}$.
\end{enumerate}
 Now, we apply Theorem~\ref{t:lower} using $k_i=w$ and $k_{i-1}=s$ to get the following lower bound with $\delta_2 = \delta/2$: 

\[\min_{S:|S|=w}\frac{\Er_\A(S)}{\textsc{OPT}_w}\geq (1-\delta_2) (w- s) \geq (w + 1) - \frac{\delta}{2} (w+ 1 + \frac{2s}{\delta} + \frac{2}{\delta}) \geq (1-\delta) (w+1),   \]

where the last inequality follows from condition 4. Also, observe that
we can replace conditions 3 and 4 with a single stronger condition: $w \geq k(1 + \frac{7}{\delta}) + 1+ \frac{2}{\delta}$.

We now iteratively construct the sequence that satisfies all of the
above conditions:
\begin{enumerate}
\item $k^u_0 = 0 $
\item For $1 \leq i \leq t$ 
  \begin{enumerate}
  \item $k^l_i =  \big\lceil\frac{7k^u_{i-1}}{\delta}  + k^u_{i-1} + \tau_\epsilon \big\rceil$.
  \item $k^u_i = \lceil k^l_i (1 + 7/\delta) + \frac{2}{\delta} + 1\rceil$. 
  \end{enumerate}
\end{enumerate}                
We can now use Theorem~\ref{t:lower} with subsequence
$\{k^u_i\}$ which also constructs the matrix $\A$ through
Lemma~\ref{l:lower}, to ensure that the lower bound of
$(1+\delta)(k^u_i+1)$ is satisfied for $\A$ for all $i$. We
can also use Theorem~\ref{t:upper} for the same matrix $\A$
and $k=k^l_i$ for any $i$ to ensure that the upper bound of
$(1+\delta)$ is also satisfied for any $i$. 
 \end{proof}

 \section{Empirical evaluation with greedy subset selection}
 \label{a:greedy}
 In this section, we provide a more detailed empirical evaluation to
 complement what we presented in Section~\ref{s:experiments}. Our aim here is to demonstrate
that our improved analysis of the CSSP/Nystr\"om approximation factor
can be useful in understanding the performance of not only the k-DPP
method, but also of greedy subset selection. Note that our theory
does not strictly apply to the greedy algorithm. Nevertheless, we show
that, similar to the k-DPP method, greedy selection also exhibits the
improved guarantees and the multiple-descent curve predicted by our
analysis.  

The most standard version of the greedy algorithm
\citep[see, e.g.,][]{Bhaskara2016GreedyCSS} 
starts with an empty set and then
iteratively adds columns that minimize the 
approximation error at every step, until we reach a set of size
$k$. 
The pseudo-code is given below.
\renewcommand{\thealgorithm}{}
\floatname{algorithm}{}
\begin{algorithm}[H]
  \caption{Greedy subset selection algorithm for CSSP/Nystr\"om}
  \begin{algorithmic}[0]
    \STATE \textbf{Input:} $k\in[n]$\quad and \quad an $m\times n$
    matrix $\A$ (CSSP), \quad or\quad an $n\times 
    n$ p.s.d.~matrix $\K=\A^\top\A$ (Nystr\"om)\vspace{2mm}
    \STATE $S\leftarrow \emptyset$
    \STATE \textbf{for} $i$ $=$ $1$ \textbf{to} $k$ \textbf{do}
    \STATE \quad Pick $i\in[n]\backslash S$ that minimizes $\Er_\A(S\cup\{i\})$,\quad or
    equivalently, $\|\K-\Kbh(S\cup\{i\})\|_*$
    \STATE \quad$S\leftarrow S\cup\{i\}$
    \STATE \textbf{end for}\vspace{2mm}
    \RETURN $S$
  \end{algorithmic}
\end{algorithm}
In our empirical evaluation we use the same experimental setup as in
Section \ref{s:experiments}, by running greedy on a toy dataset with
the linear kernel $\langle\a_i,\a_j\rangle_\text{K}=\a_i^\top\a_j$
that has one sharp spectrum drop (controlled by the condition number $\kappa$),
and two Libsvm datasets with the RBF kernel
$\langle\a_i,\a_j\rangle_\text{K}=\exp(-\|\a_i\!-\!\a_j\|^2/\sigma^2)$
for three values of the RBF parameter $\sigma$. The main question
motivating these experiments is: does the approximation factor of the
greedy algorithm exhibit the multiple-descent curve that is predicted in
our analysis, and are the peaks in this curve aligned with the sharp drops
in the spectrum of the data?

 \begin{figure*}[t]
  \centering
  \ifisarxiv
    \includegraphics[width=0.345\textwidth]{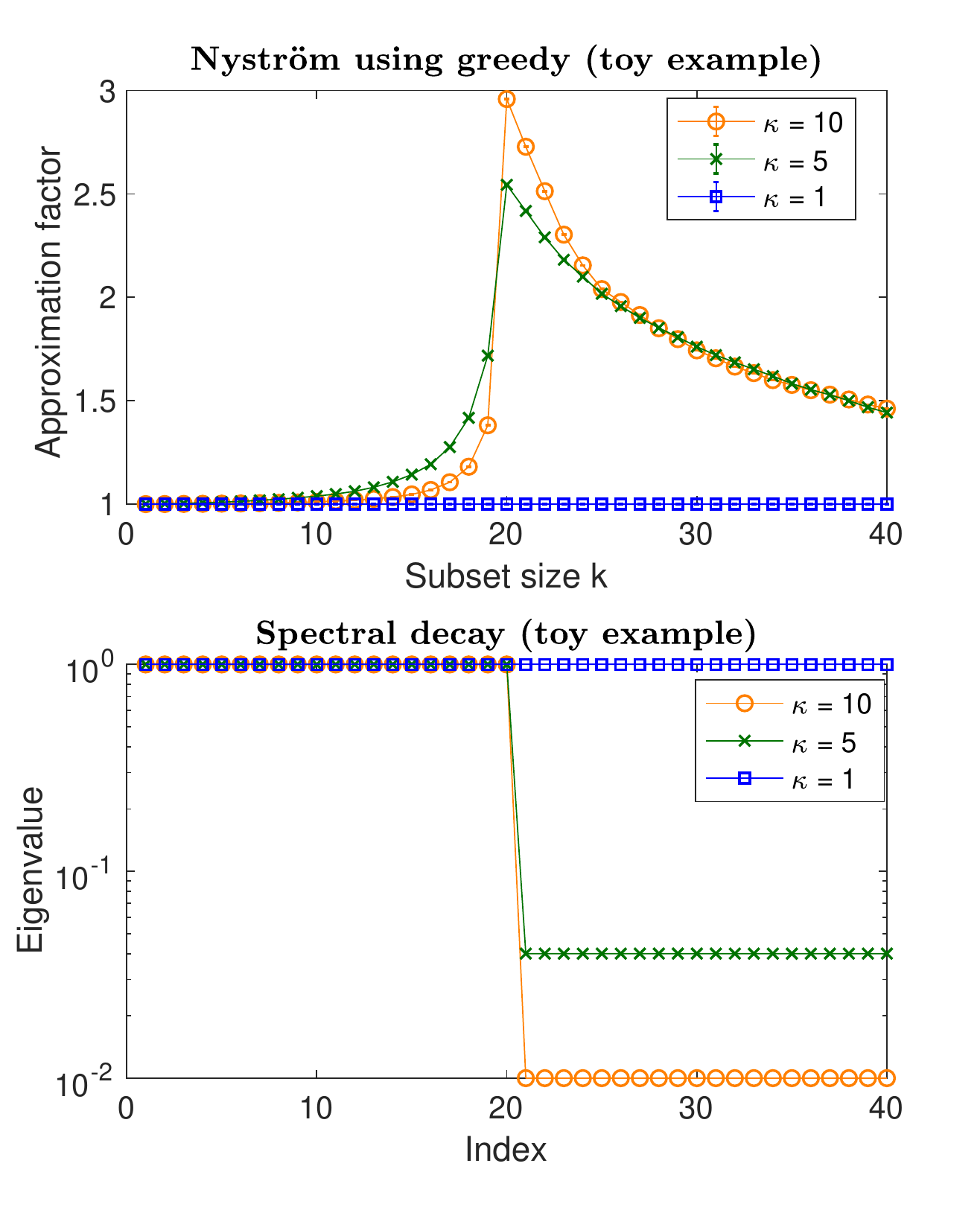}
   \hspace{-6mm}
  \includegraphics[width=0.345\textwidth]{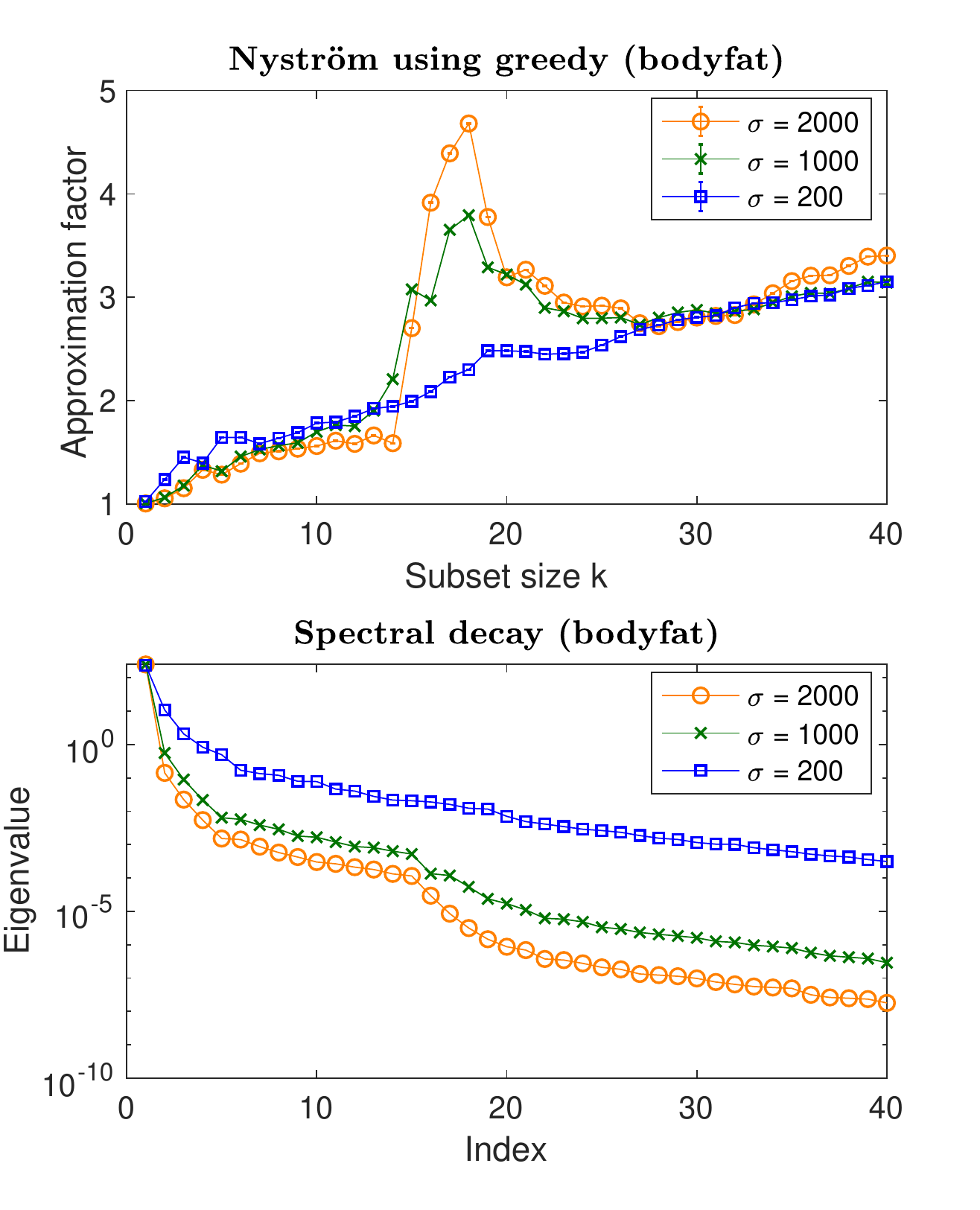}
  \hspace{-6mm}
  \includegraphics[width=0.345\textwidth]{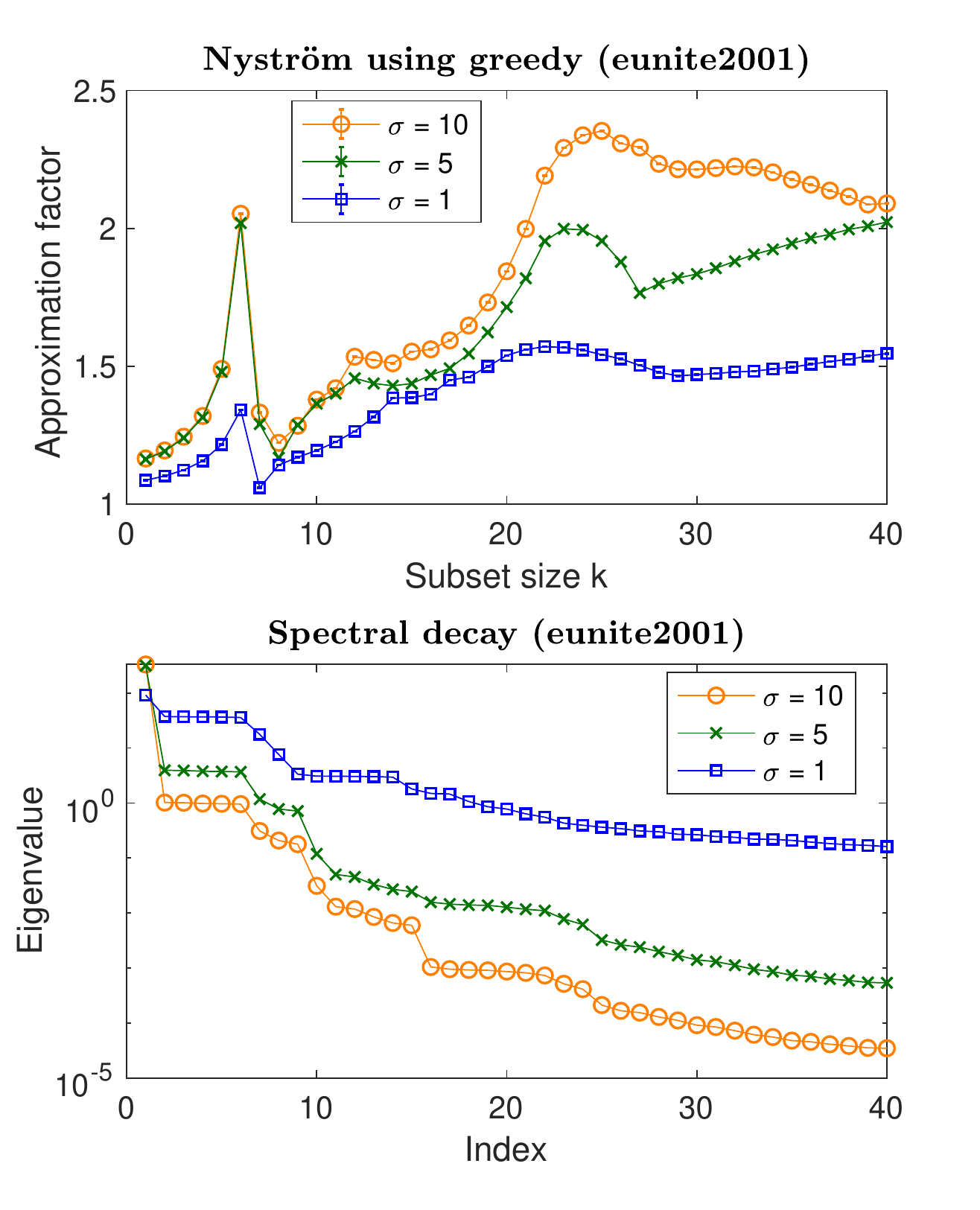}
  \else
  ~\hspace{-5mm}
  \includegraphics[width=0.35\textwidth]{figs/nystrom/rbf-toy-greedy}\hspace{-5mm}\nolinebreak\includegraphics[width=0.35\textwidth]{figs/nystrom/rbf-bodyfat-greedy}\hspace{-5mm}\nolinebreak\includegraphics[width=0.35\textwidth]{figs/nystrom/rbf-eunite2001-greedy}
  \fi
  \caption{Top plots show the Nystr\"om approximation factor
    $\|\K-\Kbh(S)\|_*/\opt$,
    where $S$ is constructed using greedy subset selection,
    against the subset size $k$, for a toy dataset
    ($\kappa$ is the condition number) and two
    Libsvm datasets ($\sigma$ is the RBF parameter). Bottom plots show
    the spectral decay for  the top $40$ eigenvalues  of each kernel
    $\K$,  demonstrating how the peaks in the Nystr\"om approximation
    factor align with the drops in the spectrum.}
\vspace{-2mm}
  \label{f:greedy}
\end{figure*}

The plots in Figure \ref{f:greedy} confirm that the
Nystr\"om approximation factor of greedy subset selection exhibits
similar peaks and valleys as those indicated by our theoretical and
empirical analysis of the k-DPP method. This is most clearly observed
for the toy dataset (Figure \ref{f:greedy} left), where the peak grows with the condition number
$\kappa$, and for the \emph{bodyfat} dataset (Figure \ref{f:greedy} center), where the size of the peak is
proportional to the RBF parameter $\sigma$. Moreover, we observe that
when the spectral decay is slow/smooth, which corresponds to smaller values
of $\sigma$, then the approximation factor of the greedy algorithm
stays relatively close to 1. For the \emph{eunite2001}
dataset (Figure \ref{f:greedy} right), the behavior of the approximation factor is very
non-linear, with several peaks occurring for large values of
$\sigma$. Interestingly, while the peaks do align with some of the
drops in the spectrum, not all of the spectrum drops result in a peak for the
greedy algorithm. This goes in line with our analysis, in the sense
that a sharp drop in the spectrum following the $k$th eigenvalue is a
\emph{necessary but not sufficient} condition for the approximation
factor of the optimal subset $S$ of size $k$ to exhibit a peak.

Our empirical evaluation leads to an overall conclusion that the
multiple-descent curve of the CSSP/Nystr\"om approximation factor is a
phenomenon exhibited by both \emph{randomized}
methods, such as the k-DPP, and \emph{deterministic}
algorithms, such as greedy subset selection. While the exact behavior
of this curve is algorithm-dependent, significant insight can be
gained about it by studying the spectral properties of the data. Our
results suggest that performing a theoretical analysis of
the multiple-descent phenomenon for greedy methods is a promising direction for future work.
 
  \end{document}

%% file: shortdefs.tex
\def\opt{{\textsc{OPT}_k}}

\def\sr{{\mathrm{sr}}}

\newcommand{\Er}{\mathrm{Er}}
\newif\ifDRAFT
\DRAFTtrue
\ifDRAFT
\newcommand{\marrow}{\marginpar[\hfill$\longrightarrow$]{$\longleftarrow$}}
\newcommand{\niceremark}[3]
   {\textcolor{red}{\textsc{#1 #2:} \marrow\textsf{#3}}}
\newcommand{\ken}[2][says]{\niceremark{Ken}{#1}{#2}}

\newcommand{\michael}[2][says]{\niceremark{Michael}{#1}{#2}}
\newcommand{\michal}[2][says]{\niceremark{Michal}{#1}{#2}}
\newcommand{\feynman}[2][says]{\niceremark{Feynman}{#1}{#2}}
\else
\newcommand{\ken}[1]{}
\newcommand{\michael}[1]{}
\newcommand{\michal}[1]{}
\newcommand{\feynman}[1]{}
\fi

\def\ee{\mathrm{e}}

\def\DPP{{\mathrm{DPP}}}

\newenvironment{proofof}[2]{\par\vspace{2mm}\noindent\textbf{Proof of {#1} {#2}}\ }{\hfill\BlackBox}

\def\Nc{\mathcal{N}}

\def\K{\mathbf K}

\def\Kbh{{\widehat{\K}}}

\ifx\BlackBox\undefined
\newcommand{\BlackBox}{\rule{1.5ex}{1.5ex}}  % end of proof
\fi

\def\a{\mathbf a}

\def\v{\mathbf v}

\def\zero{\mathbf 0}
\def\one{\mathbf 1}

\def\B{\mathbf B}
\def\A{\mathbf A}
\def\C{\mathbf C}

\def\Abh{\widehat{\mathbf A}}
\def\I{\mathbf I}

\def\A{\mathbf A}
\def\P{\mathbf P}

\def\Ot{\widetilde{O}}

\def\E{\mathbb E}
\def\R{\mathbb R}
\def\N{\mathbb N}
\def\Pr{\mathrm{Pr}}
\def\tr{\mathrm{tr}}

\def\rank{\mathrm{rank}}

\def\Var{\mathrm{Var}}

\let\origtop\top
\renewcommand\top{{\scriptscriptstyle{\origtop}}} % this makes transpose not so big

\definecolor{silver}{cmyk}{0,0,0,0.3}
\definecolor{yellow}{cmyk}{0,0,0.9,0.0}
\definecolor{reddishyellow}{cmyk}{0,0.22,1.0,0.0}
\definecolor{black}{cmyk}{0,0,0.0,1.0}
\definecolor{darkYellow}{cmyk}{0.2,0.4,1.0,0}
\definecolor{orange}{cmyk}{0.0,0.7,0.9,0}
\definecolor{darkSilver}{cmyk}{0,0,0,0.1}
\definecolor{grey}{cmyk}{0,0,0,0.5}
\definecolor{darkgreen}{cmyk}{0.6,0,0.8,0}

\ifx\proof\undefined
\newenvironment{proof}{\par\noindent{\bf Proof\ }}{\hfill\BlackBox\\[2mm]}
\fi

\ifx\theorem\undefined
\newtheorem{theorem}{Theorem}
\fi

\ifx\example\undefined
\newtheorem{example}{Example}
\fi

\ifx\condition\undefined
\newtheorem{condition}{Condition}
\fi
\ifx\property\undefined
\newtheorem{property}{Property}
\fi

\ifx\lemma\undefined
\newtheorem{lemma}{Lemma}
\fi

\ifx\proposition\undefined
\newtheorem{proposition}{Proposition}
\fi

\ifx\remark\undefined
\newtheorem{remark}{Remark}
\fi

\ifx\corollary\undefined
\newtheorem{corollary}{Corollary}
\fi

\ifx\definition\undefined
\newtheorem{definition}{Definition}
\fi

\ifx\conjecture\undefined
\newtheorem{conjecture}{Conjecture}
\fi

\ifx\axiom\undefined
\newtheorem{axiom}{Axiom}
\fi

\ifx\claim\undefined
\newtheorem{claim}{Claim}
\fi

\ifx\assumption\undefined
\newtheorem{assumption}{Assumption}
\fi

\ifx\condition\undefined
\newtheorem{condition}{Condition}
\fi